\documentclass[letterpaper]{article} 
\usepackage{aaai23}  
\usepackage{times}  
\usepackage{helvet}  
\usepackage{courier}  
\usepackage[hyphens]{url}  
\usepackage{graphicx} 
\usepackage{natbib}  
\usepackage{caption} 
\usepackage{algorithm}
\usepackage{algorithmic}
\usepackage{newfloat}
\usepackage{listings}

\usepackage[utf8]{inputenc} 
\usepackage[T1]{fontenc}    
\usepackage{url}            
\usepackage{booktabs}       
\usepackage{amsfonts}       
\usepackage{nicefrac}       
\usepackage{microtype}      
\usepackage{xcolor}         
\usepackage{enumitem}
\usepackage{algorithm} 
\usepackage{algorithmic} 
\usepackage{amssymb}
\usepackage{amsthm}
\usepackage{xcolor}
\usepackage{subcaption}
\usepackage{booktabs}
\usepackage{multirow, tabu}
\usepackage{comment}
\usepackage{svg}

\usepackage{mathtools}

\usepackage{amsmath}

\newtheorem{theorem}{Theorem}
\newtheorem{proposition}{Proposition}
\newtheorem{lemma}{Lemma}

\newtheorem{definition}{Definition}
\newtheorem{assumption}{Assumption}
\newtheorem{remark}{Remark}

\DeclareCaptionStyle{ruled}{labelfont=normalfont,labelsep=colon,strut=off} 
\floatstyle{ruled}
\newfloat{listing}{tb}{lst}{}
\floatname{listing}{Listing}
\title{On Solution Functions of Optimization: Universal Approximation and Covering Number Bounds}
\author{Ming Jin, Vanshaj Khattar, Harshal Kaushik, Bilgehan Sel, and Ruoxi Jia
}
\affiliations{Electrical and Computer Engineering, Virginia Tech
}

\usepackage{bibentry}

\begin{document}

\maketitle

\begin{abstract}
We study the expressibility and learnability of convex optimization solution functions and their multi-layer architectural extension. The main results are: \emph{(1)} the class of solution functions of linear programming (LP) and quadratic programming (QP) is a universal approximant for the $C^k$ smooth model class or some restricted Sobolev space, and we characterize the rate-distortion, \emph{(2)} the approximation power is investigated through a viewpoint of regression error, where information about the target function is provided in terms of data observations, \emph{(3)} compositionality in the form of a deep architecture with optimization as a layer is shown to reconstruct some basic functions used in numerical analysis without error, which implies that \emph{(4)} a substantial reduction in rate-distortion can be achieved with a universal network architecture, and \emph{(5)} we discuss the statistical bounds of empirical covering numbers for LP/QP, as well as a generic optimization problem (possibly nonconvex) by exploiting tame geometry. Our results provide the \emph{first rigorous analysis of the approximation and learning-theoretic properties of solution functions} with implications for algorithmic design and performance guarantees.
\end{abstract}

\section{Introduction}
\label{sec:intro}

We study the object referred to as \emph{solution function} defined by the following generic optimization:
\begin{align}\label{prob:main_prob}
    & \pi(x,\theta) = \underset{z\in R(x, \theta)}{\arg\min} \;\;g(z; x, \theta), 
\end{align}
where $g(\cdot; x, \theta):\mathbb{R}^{n_z}\to \mathbb{R}$ and $R(x, \theta)\subseteq \mathbb{R}^{n_z}$ are the objective function and feasible set (with $R:\mathbb{R}^{n_x}\times\mathbb{R}^{n_\theta}\rightrightarrows\mathbb{R}^{n_z}$ being a set-valued function), respectively, characterized by both $x\in \mathbb{R}^{n_x}$ and $\theta\in \Theta \subseteq \mathbb{R}^{n_\theta}$. We use semicolon in $g(z; x, \theta)$ to separate optimization variables from parameters. To make a further distinction, in the context of decision making, $x$ can be the input/state, $\theta$ is the parameter, and the output is the decision/action. Since the optimization solution can be a set, we make proper assumptions to ensure uniqueness \cite{dontchev2009implicit}.

Historically, the solution function (and optimal value function) has been an important basis for local sensitivity/stability and parametric analysis in optimization theory \cite{dontchev2009implicit,fiacco2020mathematical}; see \cite{amos2022tutorial} for renewed interests. However, to the best of the authors' knowledge, a mathematical theory characterizing the global properties of solution functions is still missing. Two of the basic questions are: 
\begin{enumerate}[leftmargin=2em]
    \item[\textbf{(1)}] which classes of functions can they approximate well? 
    \item[\textbf{(2)}] what is the statistical complexity of the class of solution functions? 
\end{enumerate}

The first question pertains to the expressivity of the function class. Without proper restrictions, one can easily obtain a construct with $g(z;x,\theta)\coloneqq\|z-\mu(x,\theta)\|$ and $R(x,\theta)\coloneqq\mathbb{R}^{n_z}$ so that {$\pi(x,\theta)$} can represent any function $\mu(x,\theta)$ despite being nonconvex or even discontinuous; however, such construct is neither interesting nor practically relevant. To prevent such degenerate cases, the optimization in \eqref{prob:main_prob} is, in general, assumed to be convex. In fact, as we show later in the analysis, further restrictions to LPs or QPs can still preserve the {universal approximation} property. Perhaps an even more intriguing and related question is concerning the role of ``depth'' by drawing an analogy to contemporary studies on deep neural networks (DNNs) (see, e.g., \cite{hanin2019universal,lu2021deep}). Indeed, the idea of concatenating optimizations as ``layers'' seems to be catching on \cite{AmosKolter2017,agrawal2019differentiable, kotary2021end}. Beyond the current knowledge in the multi-parametric programming literature that the solution function of an LP/QP is piecewise affine (PWA) \cite{grancharova2012explicit}, we provide a simple construction with two layers of LPs/QPs that can reproduce nonlinear functions such as product operator with no error; repeated stacking such structures by adding depth can thus reconstruct any polynomial functions of arbitrary order. Such blessings of compositionality appear in a very different form than DNNs, and will be exploited to reduce the complexity of construction (measured in terms of the number of variables and constraints). A key issue in training multi-layer compositions is the ability to backpropagate, for which we ensure compliance with the disciplined parametric programming (DPP) rules introduced in \cite{ agrawal2018rewriting}. Hence, in the first part of the study, we examine the expressivity and role of depth from the perspective of approximation theory \cite{devore1993constructive}. 
The second question has a bearing on learnability, which, ironically, has not been well-established despite recent advancements in related fields \cite{amos2022tutorial,kotary2021end,chan2021inverse,hewing2020learning}. We provide partial answers by focusing on the notion of covering numbers, which is fundamental to furnishing generalization error bounds and characterizing sample complexity \cite{cucker2007learning}. 

\subsection{Why should we study solution functions?}

Optimization is crucial in modeling complex phenomena and decision-making processes with a wide range of real-world applications \cite{boyd2004convex}. In the following, we contextualize this study by connecting to adjacent problems in machine learning, control, and operations research.

\textbf{Bi-level formulations of decision-making.} In machine learning, a lot of complex problems with a hierarchical structure are amenable to a bi-level formulation \cite{liu2021investigating}, where the inner-level solution function may correspond to a learned model (in hyperparameter optimization \cite{lorraine2020optimizing}), task-adaptive features (in multi-task and meta-learning \cite{hospedales2020meta}), attacked model (in adversarial learning \cite{zeng2021adversarial}), or critic network (in reinforcement learning \cite{hong2020two}).

\textbf{Inverse problems.} There are also a variety of problems with an \emph{inverse} nature, where the decisions are taken as input, and the goal is to infer an objective and/or constraints that render these decisions approximately or exactly optimal \cite{adams2022survey}; therein, the solution function of the corresponding optimization represents some (near-)optimal policies (in inverse reinforcement learning \cite{adams2022survey}), optimal controller (in inverse control \cite{ab2020inverse}), and Nash equilibrium of noncooperative agents (in inverse game theory \cite{bertsimas2015data,JiaKonstantakopoulosLiSpanos2018}), and identifying the parameters of the optimization is tasked to infer the hidden reward function or utility functions.

\textbf{End-to-end optimization.} The solution function can be used directly as a predictor or control policy. In decision-focused learning \cite{wilder2019melding, FeberWilderDilkinaTambe2020}, smart predict-then-optimize \cite{ElmachtoubGrigas2020,loke2021decision}, and end-to-end optimization learning \cite{kotary2021end}, a constrained optimization model is integrated into the decision-making pipeline to create a hybrid architecture, where the parameters of the solution function are trained in an end-to-end fashion (often through implicit gradients \cite{agrawal2019differentiable}). Such approaches have been demonstrated for stochastic programming \cite{donti2017task}, combinatorial optimization \cite{wilder2019melding}, and reinforcement learning \cite{wang2021learning}, with various applications in operations research (e.g., vehicle routing, inventory management, and portfolio optimization) \cite{ElmachtoubGrigas2020}, and hold the promise to enable structural inference and decision-making under constraints.

\textbf{Model-predictive control.} The solution function has a long tradition being used as a policy in model-predictive control (MPC) \cite{grancharova2012explicit}; recent advancements in learning-based MPC aim to infer the parameterization of the MPC policy, i.e., the cost and constraints, that lead to better closed-loop performance and account for safety specifications \cite{hewing2020learning}.

None of the above problems can be satisfactorily understood or solved with existing theories~\cite{dontchev2009implicit,fiacco2020mathematical}, revealing a fundamental need to study the approximation and statistical properties beyond local perturbations.

\subsection{Contributions} 
Key contributions are summarized below:

\begin{itemize}

  \item We develop a new perspective on approximation through the lens of regression with a fixed design, and establish a universal approximation theorem of the solution functions of LPs with constructive proof. The complexity of the construction is analyzed in terms of the total number of variables and constraints to obtain an $\epsilon$ accuracy (i.e., rate-distortion) (Theorem \ref{thm:approx_balaz}). 

  \item Illuminate the role of depth. Compositionality in the form of multi-layer LPs/QPs is shown to reconstruct polynomials without error (Lemma 2 in appendix). We characterize complexity with the additional depth measure and show a substantial reduction in complexity with a universal network architecture in approximating all smooth functions in some restricted Sobolev space (Theorem~\ref{thm:multilayer_n}).

  \item We discuss statistical bounds using empirical covering numbers. For LQs/QPs, we provide bounds that depend on the number of constraints and some condition number associated with constraints (Theorem \ref{thm:coveringPiecewiseAffine}). For a generic convex optimization problem, we crucially leverage the development in tame geometry by showing that the solution map is Whitney stratifiable (Theorem \ref{thm:appcovering_bound}). Our proof technique is broadly applicable to piecewise smooth functions with a bounded number of pieces. The result also has direct implications for provable convergence guarantees when training these functions with subgradient descent.
\end{itemize}

\subsection{Related work}

\textbf{Deep architecture with optimization as a layer.} Inspired by the remarkable effectiveness of DNNs, a line of work considers architectures with differentiable optimization as a layer \cite{AmosKolter2017,agrawal2019differentiable,kotary2021end}. Since conventional activation functions such as ReLU and max pooling can be reconstructed as LP solution functions, such an architecture can capture more complex structures and richer behaviors \cite{AmosKolter2017}. However, a systematic study of approximation or statistical complexity is lacking in the literature.

\textbf{Approximation and learning theory for DNNs.} The universal approximation capacity of neural networks has been well-known \cite{Hornic1989}; for instance, to achieve $\epsilon$ approximation accuracy of a ${C^k}$ smooth function {with input dimension $n_x$}, one needs $\mathcal{O}(\epsilon^{n_x/k})$ number of neurons \cite{pinkus99}.  But that alone does not explain why neural networks are so effective in practice, since functions such as polynomials, splines, and wavelets also produce universal approximants. 
Recent papers aim to elucidate this matter, with a particular focus on the role of depth \cite{Zhu2020,lu2021deep}. The role of depth has also been examined from the perspective of approximation theory \cite{yarotsky2018optimal, Chen2019}, along with various other measures such as the number of linear regions \cite{serra2018bounding} and Betti numbers \cite{bianchini2014} (see \cite{devore2021neural} for a recent survey). We also mention a very general approach to expressiveness, in the context of approximation, named the method of nonlinear widths \cite{devore1989optimal}. Existing results on statistical complexity of DNNs include bounds for the Vapnik-Chervonenkis (VC) dimension \cite{bartlett1998almost,anthony1999neural} and fat-shattering dimension \cite{anthony1999neural}, with some recent developments on tighter characterizations \cite{bartlett2019nearly}. The present study can be seen as parallel development for the solution function and its multi-layer architecture. 

\textbf{Explicit MPC and inverse optimality.} Explicit MPC exploits multiparametric programming techniques to compute the solution function offline and has been investigated for LP/QP, nonlinear convex programming, and mixed-integer programming \cite{grancharova2012explicit}. Another closely related topic studied in the control community is inverse MPC, a.k.a., inverse parametric programming, which aims to construct an optimization such that its optimal solution is equivalent to a given function \cite{baes2008every}. This inverse optimality problem has led to interesting results for general nonlinear continuous functions \cite{baes2008every} and more recently for continuous PWA functions based on techniques such as difference-of-convex (DC) decomposition \cite{hempel2014inverse} and convex lifting \cite{nguyen2018convex}. When the target function is only accessible through samples, inverse optimization can be applied to determine an optimization model that renders the set of sampled decisions approximately or exactly optimal (see, e.g., \cite{chan2021inverse} for a recent survey). Despite these developments, the questions of approximation or estimation errors of solution functions have largely eluded attention.

\textbf{Max-affine and PWA regression.} Max-affine regression, originated in \cite{hildreth1954point}, aims to recover an optimal piecewise linear approximant to a convex function through either parametric \cite{magnani2009convex,hannah2013multivariate,ghosh2019max} or nonparametric regression \cite{hildreth1954point,balazs2015near,balazs2022adaptively}. Recent studies provide theoretical guarantees on near-optimal minimax rate \cite{ghosh2019max,balazs2015near} and adaptive partitioning \cite{balazs2022adaptively}. PWA regression generalizes the function class to nonconvex candidates; an affine fit is computed separately for each partition of the space, which can be either predefined \cite{toriello2012fitting} or adaptively determined \cite{siahkamari2020piecewise}. These two lines of work are closely linked through DC modeling \cite{bavcak2011difference}.

\section{Preliminaries}

\subsection{Model class assumptions}

We consider the uniform error as $\|f-\tilde f\|_\infty = {\text{ max }}_{x\in[0,1]^{n_x}}|f(x) - \tilde f(x) |$. To provide a meaningful discussion of the approximation rate, we state the assumptions on the function class (commonly referred to as model class assumptions). Consider a Sobolev space $\mathcal{W}^{k,\infty}([0,1]^{n_x})$ defined as the space of functions on $[0,1]^{n_x}$ with derivatives up to order $k$. The norm in $\mathcal{W}^{k,\infty}([0,1]^{n_x})$ is defined as $\|h\|_{\mathcal{W}^{k,\infty}([0,1]^{n_x})} =\max_{\mathbf{k}:|\mathbf{k}|\leq k}\underset{{z}\in[0,1]^{n_x}}{\text{ess sup}} |D^\mathbf{k} h({z})|,$
here $\mathbf{k} \in \{0, 1, \dots\}^{n_x}$, $|\mathbf{k}|=\sum_{i=1}^{n_x}k_i$, and $D^\mathbf{k}\coloneqq\frac{\partial^{|\mathbf{k}|}}{\partial x_1^{k_1}\cdots\partial x_{n_x}^{k_{n_x}}}$ is the standard derivative operator. A restrictive subclass of functions $F_{k,n_x} = \left\{ h\in \mathcal{W}^{k}([0,1]^{n_x})\ :\ \|h\|_{\mathcal{W}^{k}([0,1]^{n_x})} \leq 1 \right\}$ can be considered as a unit ball in $\mathcal{W}^{k,\infty}([0,1]^{n_x})$ consisting of the functions with all their derivatives up to order $k$ bounded by unity. In addition, the class $C^k([0,1]^{n_x})$ consists of functions continuously differentiable up to order $k$. Without loss of generality, we assume the input space $X\coloneqq [0,1]^{n_x}$ and omit its dependence in the above definitions. We use $\|\cdot\|$ for the standard Euclidean norm.

\subsection{Disciplined parametrized programming}

The DPP rule, as a subset of Disciplined Convex Programming (DCP),  places mild restrictions on how parameters can enter expressions. We briefly describe the DPP rule and refer the reader to \cite[Sec. 4.1]{agrawal2019differentiable} for more details. 

Let us begin with some basic terminologies. We refer to  $x$ and $\theta$ in \eqref{prob:main_prob} as \emph{parameters}, which, once instantiated with values, are treated as constants by optimization algorithms; by contrast, $z$ is referred to as \emph{variables}, the value of which will be searched for optimal solutions. Suppose that the feasible set is defined by a finite set of constraints:
\begin{align*}
    R(x,\theta)=\{z\in\mathbb{R}^{n_z}: &\;\;g_i(z;x,\theta)\leq 0,\quad i\in[m_1]\\
    &\;\; h_i(z;x,\theta)= 0,\quad i\in[m_2]\},
\end{align*}
where we use the shorthand $[m]=\{1,...,m\}$. As in DCP, we assume that the objective function $g$ and constraints $\{g_i\}_{i\in[m_1]}$ and $\{h_i\}_{i\in[m_2]}$ are constructed from a given library of base functions, i.e., \emph{expressions}. In DPP, an expression is said to be parameter-affine if it does not involve variables and is affine in its parameters, and it is parameter-free if it does not have parameters. Under DPP, all parameters are classified as affine, just like variables. Also, the product of two expressions is affine when at least one of the expressions is constant, or when one of the expressions is parameter-affine and the other is parameter-free. For example, let $\{A,a,\lambda\}$ be parameters and $z$ be variable. Then, $Az-a=0$ is DPP because $Az$ is affine ($A$ is parameter-affine and $z$ is parameter-free), $-g$ is affine, and the sum of affine expressions is affine. Similarly, $\lambda\|z\|_2\leq 0$ is DPP because $\lambda\|z\|_2$ is affine ($\lambda$ is parameter-affine and $\|z\|_2$ is parameter-free). It is often possible to re-express non-DPP expressions in DPP-compliant ways. For instance, let $\theta_1,\theta_2$ be parameters, then $\theta_1\theta_2$ is not DPP because both arguments are parametrized; it can be rewritten in a DPP-compliant way by introducing a variable $z$, replacing $\theta_1\theta_2$ with the expression $\theta_1 z$ while adding the constraint $z=p_2$. Similarly, if $A_1$ is a parameter representing a positive semidefinite matrix, the expression $z^\top Az$ is not DPP; it can be rewritten as $\|A_2z\|_2^2$, where $A_2$ is a new parameter representing $A_1^{1/2}$. The set of DPP-compliant optimizations is very broad, including many instances of cone programs. We make sure to follow the DPP rule throughout the paper so that the result is practically relevant to end-to-end optimization that requires differentiation through the solution function for backpropagation (see \cite{agrawal2019differentiable}).

\section{Approximation through the lens of regression}
\label{sec:apprx_throu_regression}

In classical approximation theory, approximation rates are obtained assuming full access to the target function $f$ \cite{devore1993constructive,yarotsky2018optimal,devore2021neural}. In this section, we develop a new viewpoint of approximation through the lens of regression with experimental design, where we leverage an estimation procedure that learns the target function through a dataset to reason about the complexity of approximation.

We formulate the approximation problem in a setting closely related to fixed-design regression \cite{gyorfi2002distribution}. Here, let $\mathcal{D}_n=\{(x_1,f(x_1)),\cdots,(x_n,f(x_n))\}$ be a dataset of $n\in \mathbb{N}$ points, where the locations $\{x_i\}_{i\in[n]}\in X^n$ to evaluate the target function $f$ can be arbitrarily selected.\footnote{Note that, for the purpose of analyzing approximation power, the labels received in the dataset are assumed to be accurate (i.e., noiseless observations). Noisy data can be processed by combining the proposed method with standard regression techniques. } The key idea of our proof technique is to first construct an estimator $\mathcal A:(X\times\mathbb{R})^n\to \Pi$, where $\Pi$ is the class of functions that we are analyzing (i.e., the set of solution functions in our setting).\footnote{In general, we can allow the estimator to have infinite computational power to solve a nonconvex optimization to arbitrary accuracy; however, for practical purposes, we restrict it to being a computationally efficient procedure so that we can obtain an approximant with a reasonable amount of time by learning from a finite dataset $\mathcal{D}_n$.} We then characterize the approximation error based on the regression error of the estimator; meanwhile, we can reason about the rate-distortion by examining the complexity of the constructed function $\mathcal A(\mathcal{D}_n)$.

In the following theorem, we establish the \emph{first} universal approximation theorem for the class of solution functions corresponding to LPs. Readers are referred to the supplementary material for details of proof in the main document.

\begin{theorem}[Approximation of $C^2$ by max-affine regression]\label{thm:approx_balaz}
 For any target function  $f\in C^2$ and $\epsilon>0$, there exists a solution function $\pi$ of an LP with $\mathcal{O}\Big(\left(\frac{n_x}{\epsilon}\right)^{\frac{n_x}{2}}\Big)$ constraints and $n_x+1$ variables, such that 
    $\|f-\pi\|_\infty\leq \epsilon.$
\end{theorem}
The proof exploits the fact that any $C^2$ function can be approximated by a DC function \cite{bavcak2011difference}; we then construct a numerical procedure by extending the algorithm from \cite{balazs2015near} for max-affine regression to learn the potentially nonconvex target function from some dataset $\mathcal{D}_n$.

Note that in statistical learning theory, it is uncommon to impose a model class assumption on the target function that gives rise to the data, so the generalization error is compared with the best-in-class; while the generalization error may vanish as more data is collected, the approximation error may always be bounded away from zero because the target function may not lie in the function class of estimators \cite{cucker2007learning}. The critical implication of the above result is that we can approximate any smooth function to arbitrary precision by constructing an LP with enough constraints and variables. This is not without surprise, as LPs have arguably the simplest form within the broad classes of optimization \cite{boyd2004convex}. 

If we count a ``neuron'' in a neural network the same way we count a constraint in optimization, then the above approximation scheme gives the same order of complexity in terms of $\epsilon$ as a one-layer neural network \cite{devore2021neural}; however, the authors admit that a head-to-head comparison may not be fair (indeed, we later refer to an entire optimization program as a generalized neuron). Interestingly, complexity is mainly reflected in the number of constraints; the number of variables can be kept at the same level as the input dimension. Lastly, it is not our intention to exhaust all possible construction methods to derive the rate-distortion; other methods may also apply \cite{he2020relu,devore2021neural}.

\section{The role of depth}\label{sec:role_depth}

In this section, we illuminate  the role of depth in using deep architectures with solution functions. Interestingly, depth plays a very different role herein compared to DNNs \cite{yarotsky2018optimal,devore2021neural}. We begin by introducing some formalities to characterize the architecture, which may be of independent interest.

\subsection{Optimization-induced network architecture}

We consider a deep network as a directed acyclic graph (DAG), $\mathcal{N} \coloneqq (\mathcal{V}, \mathcal{E})$, where $\mathcal{V}$ and $\mathcal{E}$ are finite sets of vertices (a.k.a., nodes) and directed edges. The set $\mathcal{V}$ consists of the set $\mathcal{V}_i$ of input vertices as placeholders for independent variables (i.e., inputs $x$), the set $\mathcal{V}_o$ of output vertices (i.e., corresponding output), and the set $\mathcal{V}_h\coloneqq\mathcal{V}\setminus\{\mathcal{V}_i,\mathcal{V}_o\}$ of hidden vertices, which store certain intermediate values to compute the output. The output of each $v\in\mathcal{V}\setminus\mathcal{V}_i$ is given by the solution function (parameterized by $\theta_v$), $\pi_v(\cdot;\theta_v)$, which takes as input from incoming edges; for each edge $e\in\mathcal{E}$, an affine transformation $h_e(\cdot;\theta_e)$ is applied to the output of the incident node. Analogous to DNNs, we define a general notion of a neuron as a computational unit associated with each node $v\in\mathcal{V}$, which takes as input the (possibly vector-valued) outputs $x_{v'}$ from the incident nodes $v'\in\mathcal{V}\setminus\mathcal{V}_o$ with an edge $e=(v',v)\in\mathcal{E}$ directed to $v$, and produces the output
\begin{equation}
    x_v\coloneqq \pi_v(\{h_e(x_{v'};\theta_e)\}_{e=(v',v)\in\mathcal{E}},\theta_v).
\end{equation}
As a convention, outputs from input nodes $v\in\mathcal{V}_i$ are externally provided function inputs; outputs from neurons associated with output nodes $v\in\mathcal{V}_o$ are given by affine transformation of values from adjacent incoming nodes. Thus, we define the output function $f_\mathcal{N}:\mathbb{R}^{n_x}\rightarrow\mathbb{R}^{n_o}$ of the network $\mathcal{N}$ by
\begin{equation}\label{eq:output-N}
    f_\mathcal{N}(x)\coloneqq(x_v,v\in\mathcal{V}_o).
\end{equation}
Since a vector-valued function can be regarded as a concatenation of scalar-valued functions, for simplicity, we will only consider the case where $n_o=1$.
The collection of $\{\theta_v,\theta_e\}$ for ${v\in\mathcal{V}}$ and ${e\in\mathcal{E}}$ are referred to as the trainable parameters of $\mathcal{N}$. For a fixed architecture, the set of output functions forms a parameterized nonlinear manifold.

For the exposition, we can also organize the nodes of $\mathcal{N}$ into layers. The zeroth layer, called the input layer, consists of all $n_x$ input vertices in $\mathcal{V}_i$. The input layer is followed by hidden vertices organized into $L$ hidden layers, with the $j$-th layer $\mathcal{H}_j$ consisting of all $n_j$ vertices that are $j$-hop away from the input layer excluding the output vertices, for $j\in[L]$. Finally, the output layer consists of all output nodes $\mathcal{V}_o$, which contains at least one node that is $L+1$-hop neighbor of the input layer (otherwise, the depth must be less than $L$). The main distinction with conventional DNNs is that the computation of a neuron is given by some solution function instead of the usual coordinate-wise activation function (e.g., ReLU). Since the computational complexity of an optimization family (e.g., LPs/QPs) can usually be characterized by the number of variables and constraints, we measure the width of each layer by the total number of variables and constraints among nodes therein, due to a simple fact that we state without proof.
\begin{proposition}[Concatenation rule]\label{prop:concat}
The concatenation of solution functions $\{\pi_v(\cdot,\theta_v)\}_{v\in\mathcal{V}'}$, where $\pi_v(\cdot,\theta_v)$ is associated with an optimization with $n^z_v$ variables and $n^c_v$ constraints, can be written as a solution function of some optimization with $\sum_{v\in\mathcal{V}'}n^z_v$ variables and $\sum_{v\in\mathcal{V}'}n^c_v$ constraints (up to some additive constants no larger than $|\mathcal{V}'|+1$).
\end{proposition}
Henceforth, we refer to the integers $W_j^v\coloneqq\sum_{v\in\mathcal{H}_j}n^z_v$ and $W_j^c\coloneqq\sum_{v\in\mathcal{H}_j}n^c_v$ as the \emph{variable width} (v-width) and \emph{constraint width} (c-width) of the $j$-th layer, respectively, and $L$ as the \emph{depth} of the network. We usually use the maximum v-width and c-width among all hidden layers, denoted by $W^v$ and $W^c$ respectively, to characterize the network width. 

\begin{definition}
We define $\Upsilon^{W^v,W^c,L}$ as the family of functions $f_\mathcal{N}$ in \eqref{eq:output-N} with width and depth bounded by $W^v,W^c$, and $L$.
\end{definition}

The set of solution functions $\Pi = \left\{ \pi(\cdot; \theta):\mathbb{R}^{n_x}\to\mathbb{R}\ |\ \theta\in \Theta \right\}$ can be regarded as a single-layer network where the output transformation is identity. In the sections pertaining to approximation property, we will focus exclusively on LP/QP solution functions. Also, to make a distinction between \emph{networks} and \emph{network architectures}: We define the latter as the former with unspecified weights. The universal approximation property of a network architecture is discussed in the sense that we can approximate any function from a model class $\mathcal{F}$ with error $\epsilon$ by simply some weight assignment.

\subsection{Exact construction of Taylor polynomials}
In the recent work of \cite{yarotsky2018optimal}, it is shown that DNNs can approximate some elementary functions such as the multiplication operator with progressive accuracy by increasing depth, which is then used to establish the improvement of approximation due to depth. The first intriguing role of depth for optimization-induced architecture is that we can \emph{exactly construct} the product function with a network of only two layers. This is based on the following simple observations:
\begin{itemize}[leftmargin=*]

  \item The solution to $\{\min_{z} -z \;\;\text{s.t.} \;\;x_1z\leq x_2 \}$, where $x_1,x_2\in (0,1]$ are the parameters of the optimization (i.e., inputs), is exactly $x_2/x_1$.

  \item A two-layer architecture, with $\{\min_{z} -z \;\;\text{s.t.} \;\;x_1z\leq 1 \}$ as the first layer, the output of which is provided as input to $\{\min_{z} -z \;\;\text{s.t.} \;\;\square z\leq x_2 \}$ as the parameter value for $\square$, has an output of $x_1x_2$ for any $x_1,x_2\in(0,1]$.

\end{itemize}
Both observations can be directly verified by writing the Karush–Kuhn–Tucker (KKT) conditions.\footnote{The first observation, in particular, may be contradictory to the common belief held in the multi-parametric programming literature that the solution function of an LP is always piecewise linear \cite{grancharova2012explicit}. A close examination of their argument reveals that it holds true only when the LP is free of any expression in its constraints where parameters multiply with variables. 
}  
Note that in the above, we do not consider the measure-zero event that any coordinate of $x$ can be 0. Most importantly, observe that the above constructions comply with the DPP rule (e.g., $x_1z$ is affine because $x_1$ is parameter-affine and $z$ is parameter-free). In the following, we use $\begin{pmatrix}k\\n\end{pmatrix}$ to denote the binomial coefficient $n$ choose $k$ and $\lceil\cdot\rceil$ as the ceiling function.

\begin{theorem}\label{thm:multilayer_n}
There exists a universal multilayer LP architecture that:
    \begin{enumerate}[label=(\alph*),leftmargin=2em]
        \item[(a)] can approximate any function from $F_{k, n_x}$ with uniform error bounded by $\epsilon>0$;
        \item[(b)] has a depth of at most {$2k$} and the widest layer has at most { $\bigg( 2n_x+2+2\begin{pmatrix}n_x\\k+n_x\end{pmatrix}\bigg)N^{n_x}$ constraints and $\bigg(1+\begin{pmatrix}n_x\\k+n_x\end{pmatrix}\bigg)N^{n_x}$ variables, where $N=\lceil n_x (\frac{1}{k!\epsilon})^{1/k} \rceil$.}
    \end{enumerate}
\end{theorem}

The theorem above provides an upper bound for approximation complexity with a universal network architecture to approximate all functions in $F_{k, n_x}$. For DNNs to achieve the same error $\epsilon$ \cite{yarotsky2017error}, one needs a depth of $\mathcal{O}(\ln(1/\epsilon))$ with $\mathcal{O}(\epsilon^{-n_x/k}\ln(1/\epsilon))$ weights; in contrast, for optimization-induced networks, we only need a fixed depth that does not grow with the accuracy requirement $\epsilon$ and $\mathcal{O}(\epsilon^{-n_x/k})$ constraints and variables. The removal of the dependence of depth on $\epsilon$ is due to the fact that instead of approximating some Taylor polynomial, as is done for DNNs with ReLU activation, we are able to \emph{exactly reconstruct} the polynomial function with a deep optimization-induced network, which eliminates any approximation error due to  reconstruction. We can also  remove a factor of $2^{n_x^2/k}$ from the number of constraints and variables relative to what would yield without exploiting the special advantage of solution functions. This is based on another simple observation:
\begin{itemize}[leftmargin=*]
    \item The solution to $\{\min_{z\in[0,1]} -z \;\;\text{s.t.} \;\;(x-1)z\leq 0, (x+1)z\geq 0\}$, where $x\in \mathbb{R}$ is the optimization parameter (i.e., input), is exactly the bump function $$\psi(x)=\begin{cases}
  1,&|x|\leq 1\\
  0,&\text{otherwise}
\end{cases};$$
    \item More generally, for an arbitrary union of intervals $\cup_{i\in I}[a_i,b_i]$, the solution to $\{\min_{z\in[0,1]} -z \;\;\text{s.t.} \;\;(x-b_i)z\leq 0, (x-a_i)z\geq 0, \forall i\in I\}$, where $x\in \mathbb{R}$, is exactly the multi-bump function that is 1 if $x\in \cup_{i\in I}[a_i,b_i]$ and 0 otherwise.
\end{itemize}
The proof relies on partitioning the input space into a grid of $(N+1)^{n_x}$ functions. With the above bump functions, we can decrease the size of the grid $N$ from $\lceil(\frac{\epsilon k!}{2^{n_x}n_x^k})^{-1/k}\rceil$ in \cite{yarotsky2018optimal} to $\lceil(\frac{\epsilon k!}{n_x^k})^{-1/k}\rceil$ by a factor of $2^{n_x/k}$, resulting in an overall reduction of $2^{n_x^2/k}$ in complexity, which can be substantial for high dimensions. Interestingly, from the construction of bump functions, the solution function can act as some switching mechanism or logical expressions (e.g., if-then-else), similar to the role of a binary variable in mixed-integer programming. {To sum up, Sections \ref{sec:apprx_throu_regression} and \ref{sec:role_depth} answered the first question posed in the introduction, namely, the approximation properties of the solution functions. We now proceed to answer the second question, namely, the statistical complexity of the solution functions.}




\section{Definability and Whitney stratification}\label{sec:definability}

Real-world optimization often has some nice structures that can be exploited \cite{ioffe2009invitation}. Given some mild assumptions about objective/constraint functions, we can show that the solution function enjoys a nice property, namely, whitney stratification, which induces desirable computational guarantees \cite{ioffe2009invitation,davis2020stochastic}. In this section, we resume the generality of a convex optimization problem.

First, let us recall some fundamental concepts in tame geometry \citep{van1996geometric,ioffe2009invitation}.
\begin{definition}[Whitney Stratification]
\label{def:WhitneyStratification-main}
 A Whitney $C^k$ stratification of a set $I$ is a partition of $I$ into finitely many nonempty $C^k$ manifolds, called strata, with the following conditions:
\begin{enumerate}
    \item[1)] For any two strata $I_a$ and $I_b$, $I_a \cap I_b \neq \emptyset$ implies that $I_a \subset \mathrm{cl} I_b$ holds, where $\mathrm{cl} I_b$ is the closure of the set $I_b$. 
    \item[2)] For any sequence of points $x_k$ in a stratum $I_a$, converging to a point $x^{\star}$ in a stratum $I_b$, if the corresponding normal vectors $v_k \in N_{I_a}(x_k)$ converge to a vector $v$, then the inclusion $v \in N_{I_b}(x^{\star})$ holds. Here, $N_{I_a}(x_k)$ denotes the normal cone to $I_a$ at $x_k$.
\end{enumerate}
\end{definition}

Roughly speaking, stratification is a locally finite partition of a given set into differentiable manifolds, which fit together in a regular manner (property \emph{1)} in Def. \ref{def:WhitneyStratification-main}). Whitney stratification as defined above is a special type of stratification for which the strata are such that their tangent spaces (as viewed from normal cones) also fit regularly (property \emph{2)}). There are several ways to verify Whitney stratifiability. For example, one can show that the function under study belongs to one of the well-known function classes, such as semialgebraic functions, whose members are known to be Whitney stratifiable \citep{van1996geometric}. However, to study the solution function of a general convex optimization problem, we need a far-reaching axiomatic extension of semialgebraic sets to classes of functions definable on ``o-minimal structures,'' which are very general classes and share several attractive analytic features as semialgebraic sets, including Whitney stratifiability \citep{van1996geometric}; the definition of o-minimal structures can be found in the appendix.
\begin{assumption}\label{asmptn:Definable}
The function $g$ and the set-valued map $R$ are definable in some o-minimal structure.
\end{assumption}

This is  a mild assumption as practically all functions from real-world applications, including deep neural networks, are definable in some o-minimal structure \cite{davis2020stochastic}; also, the composition of mappings, along with the sum, inf-convolution, and several other classical operations of analysis involving a finite number of definable objects in some o-minimal structure remains in the same structure \cite{van1996geometric}.
\begin{theorem}\label{prop:definability_pi}
The solution function \eqref{prob:main_prob} is Whitney stratifiable. In addition, any function in the class $\Upsilon^{W^v,W^c,L}$ is Whitney stratifiable for any positive integers $W^v,W^c$, and $L$.
\end{theorem}

The far-reaching consequence of definability, exploited in this study, is that definable sets and functions admit, for each $k \geq 1$, a $C^k$–Whitney stratification with finitely many strata (see, for instance, \cite[Result 4.8]{van1996geometric}). This remarkable property, combined with the result that any stratifiable functions enjoy a nonsmooth Kurdyka–\L{}ojasiewicz inequality \cite{bolte2007clarke}, provides the basis for convergence analysis of many optimization algorithms \cite{drusvyatskiy2018error}. In particular, the application of subgradient methods to solution functions or optimization-induced networks is endowed with rigorous convergence guarantees (see, e.g., \cite{davis2020stochastic}).

\section{Covering number bounds}\label{sec:statistical_bnds}

In this section, we provide covering number bounds for the solution functions of LPs, QPs, and in general, any convex programs. We focus on the empirical $L_1$ covering number $\mathcal{N}_1(\epsilon,\Pi,n)$, which is a variant of the covering number for a set of functions $\Pi$ at $\epsilon$ accuracy with respect to the empirical $L_1$ metric defined over $n$ data points. 

\begin{definition}[\cite{zhang2002covering}, empirical $L_1$ covering number] 
Given observations $\mathcal{D}_n=\{x_1,...,x_n\}$ and vectors $f(\mathcal{D}_n)=[f(x_1),...,f(x_n)]\in\mathbb{R}^n$ for any $f\in\mathcal{F}$, the empirical $L_1$ covering number, denoted as $\mathcal{N}_1(\epsilon,\mathcal{F},\mathcal{D}_n)$, is the minimum number $m$ of a collection of vectors  $g_1,...,g_m\in\mathcal{F}$, such that $\forall f\in\mathcal{F}$, there exists an $g_j$ such that 
\begin{equation*}
    \|f-g_j\|_{\mathcal{D}_n}\coloneqq \frac{1}{n}\sum_{i=1}^n |f(x_i)-g_j(x_i)|\leq\epsilon.
\end{equation*}
We define $\mathcal{N}_1(\epsilon,\mathcal{F},n)=\sup_{\mathcal{D}_n}\mathcal{N}_1(\epsilon,\mathcal{F},\mathcal{D}_n)$. The set $\{g_1,...,g_m\}$ above is called the (empirical) $\epsilon$-cover of $\mathcal{F}$ and the logarithm of covering number $\log \mathcal{N}_1(\epsilon,\mathcal{F},n)$ is known as the entropy number.
\end{definition}

\subsection{The class of LPs and QPs} Consider the function  
\begin{equation}
    \label{equ:qp}
\pi_{QP}(x;\theta) \coloneqq \underset{z\in R(x,\theta)}{\arg\min} \left(\frac{1}{2}A_0 z+U_0^x x+U_0^\theta \theta +b_0\right)^\top z,
\end{equation}
with $R(x,\theta)\coloneqq\{z: A_{1}z\leq b_{1}+U_{1}^xx+U_{1}^\theta\theta,A_{2}z = b_{2} + U^x_{2} x+U^\theta_{2} \theta\}$, where $x$ and $\theta$ are the parameters, $z\in\mathbb{R}^{n_z}$ is the optimization variable, and all the rest are fixed hyperparameters of compatible dimensions; in particular, $A_0\succ 0$ is positive definite, and $m_{1}$ and $m_{2}$ are the number of inequality and equality constraints, respectively. We can also define $\pi_{LP}(x;\theta)$ by setting $A_0= 0$ in \eqref{equ:qp}. Now, let us also introduce 
\begin{equation}
\label{eq:kappa-lp}
    \kappa_{LP}^\star=\max_{\iota\subseteq\{1,...,m_1\},\lvert\iota \rvert \leq n_z-m_2}\left\|\tilde{A}_{LP}(\iota)\right\|_2,
\end{equation}
where 
\begin{equation}\label{eq:A-LP}
    \tilde{A}_{LP}(\iota)=\begin{bmatrix}
    [A_{1}]_{\iota}\\
    A_{2}
    \end{bmatrix}^{-1}\begin{bmatrix}
    [U_{1}^x]_{\iota}\\
    U_{2}^x
    \end{bmatrix}.
\end{equation}
Similarly, define
\begin{equation}
    \label{eq:kappa-qp}
    \kappa_{QP}^\star=\max_{\iota\subseteq\{1,...,m_1\},\lvert\iota \rvert \leq n_z-m_2}\|\tilde{A}_{QP}(\iota)\|_2,
\end{equation}
where $\tilde{A}_{QP}(\iota)$ is given as
    \begin{equation}
        \label{eq:A-qp}
\tilde{A}_{QP}(\iota)=M\left([A_1]_{\iota}A_0^{-1}U_0^x+[U_1^x]_{\iota}\right)
        -A_0^{-1}[U_0^x]_{\iota}
    \end{equation}
    with
    $M=A_0^{-1}[A_1]_{\iota}^\top\left([A_1]_{\iota}A_0^{-1}[A_1]_{\iota}^\top\right)^{-1}$. Here, $\|\cdot\|_2$ is the spectral norm, and  $[A_{1}]_{\iota}$ is the submatrix formed by the rows of $A_{1}$ indexed by $\iota$, and the inverse is understood as pseudo-inverse in the case of a rectangular matrix. The quantities $\kappa_{LP}^\star$ and $\kappa_{QP}^\star$ are some condition numbers associated with the optimization parameters. For instance, under the linear independence constraint qualification \cite{LuoPangRalph1996}, the inverse matrix in \eqref{eq:A-LP} is always full-rank; the alignment of some constraints, on the other hand, will result in larger values of $\kappa_{LP}^\star$, since it will yield larger dual variables. In this section, we assume standard constraint qualifications,  including  Mangasarian-Fromovitz constraint qualification,  constant rank constraint qualification, and strong coherent orientation condition \cite{LuoPangRalph1996}, such that both $\kappa_{LP}^\star$ and $\kappa_{QP}^\star$ are bounded. We also assume that  $\Theta$ is compact and $m_1+m_2\geq n_z$, where $n_z$ is the dimension of the variables in \eqref{equ:qp}; this assumption is easy to be relaxed with a slightly more complicated (but not necessarily more insightful) bound, thus we make the restriction to streamline the presentation. We define $\Pi_\square=\{\pi_\square(\cdot,\theta):\theta\in\Theta\}$, where $\square$ can be LP or QP.

\begin{theorem}
\label{thm:coveringPiecewiseAffine}
The empirical $ L_1$ covering number of $\Pi_\square$ over bounded input space is controlled by
\begin{equation*}
\begin{aligned}
\label{eq:LPcoveringBound}
   \log \mathcal{N}_1(\epsilon,\Pi_\square,n) \lesssim \frac{\kappa_\square^{\star 2}}{\epsilon^2} {\sum}_{0\leq i\leq {n_z-m_2}}\begin{pmatrix}m_1\\ i\end{pmatrix},
\end{aligned}
\end{equation*}
where  $\square$ can be either LP or QP.
\end{theorem}
The above bound implies that the complexity of the class of LPs and QPs increases by the number of inequality constraints (which agrees with our approximation results obtained so far) and also depends on the conditional numbers $\kappa_{LP}^\star$ and $\kappa_{QP}^\star$ that bound the maximum slope of affine functions among all pieces.

\subsection{Generic optimization class} 
More generally, we consider any optimization (possibly nonconvex) with a definable objective and constraints. We note that bounding the entropy numbers of the classes $C^k$ with respect to the supremum norm was were among the first results after the introduction of the concept of covering numbers (e.g., \cite{cucker2002mathematical}). However, the solution function does not belong to $C^k$ since the function is only piecewise smooth and may be even not differentiable at the boundary between two pieces; besides, each piece may be nonconvex, but existing results assume that the domain is convex. 
\begin{theorem}
\label{thm:appcovering_bound}
Consider the set $\Pi\coloneqq\{\pi(\cdot,\theta):\theta\in\Theta\}$, where $\pi$ is defined in \eqref{prob:main_prob} and  $X$ and $\Theta$ are compact. Then,
\begin{equation*}
    \label{eq:finalCoveringBoundSmooth-app}
     \log \mathcal{N}_1(\epsilon, \Pi, n)  \lesssim n\left({1}/{\epsilon} \right)^{1/k}+k^{n_x}\log\left({1}/{\epsilon} \right),
\end{equation*}
where  the constant depends on the number of strata in the $C^k$-Whitney stratification, which is always finite.
\end{theorem}
The proof exploits the result from the last section that the solution function of any optimization given by definable objective and constraints is Whitney stratifiable. This provides us with a starting point to bound the complexity, since any Whitney stratifiable function is piecewise-smooth with bounded pieces~\cite{van1996geometric}. We also proved an empirical $L_1$ covering number bound for $C^k$ functions, which can be of independent interest. Note that \cite{pontil2003note} proved that the empirical covering number is on the same order of the standard covering number for the smooth function class, in the sense that the empirical covering number can be lower bounded by the standard covering number on a larger scale (see the lower bound in \cite[Thm. 1]{pontil2003note}). However, for their lower bound to be non-vacuous, we need the size of the dataset to be on the same order as a covering set of the entire space. This only applies when the number of data scales exponentially with dimension $n_x$, which is rarely the case with practical problems. On the other hand, our result makes explicit the dependence on the number of data points $n$ and recovers the bound for the standard covering number when $n=\mathcal{O}\left(({1}/\epsilon)^{n_x}\right)$ in view of \cite{pontil2003note}. 

To summarize this section, the key implication of Theorems \ref{thm:coveringPiecewiseAffine} and \ref{thm:appcovering_bound} is that the solution function of any definable optimization is statistically learnable.

\section{Conclusion}\label{sec:conclude}
In this paper, we provide definite (but partial) answers to fundamental questions about the approximation and learning-theoretic properties of solution functions. The results provided in the paper can be used to understand questions about sample complexity and approximation capacity in the practice of decision making, and can help guide practice in various engineering domains discussed in the Introduction. Given the importance of this class of functions in the practical and theoretical arena, we expect our results to advance the understanding of algorithm design and spur further research on this problem.

\section{Acknowledgments}
The authors acknowledge the generous support by NSF, the Commonwealth Cyber Initiative (CCI), C3.ai Digital Transformation Institute, and the U.S. Department of Energy.

\bibliography{references}

\newpage
\appendix
\onecolumn

\section{Appendix}

\section{Proofs for approximation through the lens of regression \ref{sec:apprx_throu_regression}} \label{app:apprx_throu_regression}

We begin with some notations. Following \cite{balazs2015near,balazs2022adaptively}, we introduce the class of convex, bounded, subdifferentiable, and uniformly Lipschitz functions on the set ${X}$ as the following 
 \begin{align}\label{eqn:C_Balaz}
    \mathcal{C}_{{X}, B, L}\coloneqq \left\{ f:{X} \to \mathbb{R} \bigg |\ f \text{ is convex}, \| f\|_\infty \leq B,  \text{and } \|s\|_\infty \leq L,\ \forall s\in \partial  f(x)  \right\},
\end{align}
with scalars $B,L>0.$ We also introduce the class of max-affine functions that are uniformly bounded and uniformly Lipschitz with at most $K\in \mathbb{N}$ hyperplanes:
\begin{align}\label{eqn:M_Balaz}
      \mathcal{M}_{{X}, B, L}^K\coloneqq \left\{ h:{X} \to \mathbb{R} \bigg |\ h(x) = \underset{k = 1, \dots, K}{\text{max }} p_k^\top x+ q_k,\   \|p_k\|_\infty \leq L,\  h(x) \in [-B_d, B],\ \forall x\in X  \right\},
\end{align}
where $B_d\coloneqq B+n_x L$. We also denote $\mathrm{diam}(X')\coloneqq\max_{x,x'\in X'}\|x-x'\|_\infty$ as the diameter of the set $X$. For example, $\mathrm{diam}(X)\coloneqq 1$ by the assumption that $X\coloneqq [0,1]^{n_x}$. We start with a simple observation with nontrivial implications.

\begin{lemma}
\label{lem:max_affine_reconstruct}
Any function $h\in \mathcal{M}_{{X}, B, L}^K$ can be written as a solution function of some LP with $K$ constraints and $n_x+1$ variables. In addition, any function of the form $f=h_1-h_2$, where $h_1,h_2\in \mathcal{M}_{{X}, B, L}^K$, can be written as a solution function of some LP with $2K+1$ constraints and $n_x+3$ variables.
\end{lemma}
\begin{proof}
Suppose $h(x) = \underset{k = 1, \dots, K}{\text{max }} p_k^\top x+ q_k$. It can be seen that by introducing an additional variable $t$ and $K$ constraints in the form of $p_k^\top x+ q_k\leq t$ and changing the maximum to minimum, we have constructed an equivalent optimization with solution equal to $h(x)$. This is known as the epigraph reformulation of an optimization. The construction for $f=h_1-h_2$ is performed by introducing $t_i$ and $K$ constraints for each $h_i$, $i=1,2$, plus an additional variable $t_3$ and an additional constraint $t_1+t_2\leq t_3$, with the objective to minimizes over $t_3$.
\end{proof}

\subsection{Proof of Theorem \ref{thm:approx_balaz}}\label{subsec:approx_balaz}


\begin{proof}

By \cite{bavcak2011difference}, any function $f\in C^2$ can be written as a DC function: $f = \phi_1 - \phi_2$, where $\phi_i\in C_{{X},B,L}$  for some $B$ and $L$. The following proof extended the algorithm of \cite{bronshteyn1975approximation,balazs2015near}, which is designed for convex max-affine functions, to the case of piecewise affine functions (not necessarily convex). For any $x \in X$ and convex function $\phi$, let $\nabla \phi(x)\in\partial \phi(x)$ be an arbitrary fixed subgradient of $\phi$ at $x$. For any $t>0$ and $i=1,2$, define $R_t\coloneqq 1+2tL$, $\nu_i(x)\coloneqq x+t\nabla \phi_i(x)$  that combines the point $x$ and $\nabla \phi_i(x)$ weighted by $t$, and  $\mathcal{K}_i\coloneqq\{\nu_i(x):x\in X\}\subset \mathbb{R}^{n_x}$ as an expanded set of $X$ along the direction $\nabla \phi_i$. Note that since the subgradient of a convex function is monotone, $\nu_i(\cdot)$ is strictly monotone, so $\nu_i(x)\neq \nu_i(y)$ for any $x\neq y$. This also implies that $\nu_i(\cdot)$ is a bijection and its inversion is well-defined. Let $\mathcal{K}_{\epsilon,i}\subseteq \mathcal{K}_i$ be an $\sqrt{\epsilon}$-net of set $\mathcal{K}_i$ with respect to Euclidean norm $\|\cdot\| $, and ${X}_{\epsilon,i} \triangleq \left\{  \nu_i^{-1}(z) \in {X} : z\in \mathcal{K}_{\epsilon,i} \right\}$ be its preimage corresponding to the mapping $\nu_i$ for $i=1,2$. Since $R_t\geq \mathrm{diam}(\mathcal{K}_i)$, by standard covering number argument \cite{wainwright2019high} and the fact that $\|x\|\leq \sqrt{n_x}\|x\|_\infty$ for any $x$, $|\mathcal{K}_{\epsilon,i}|=|X_{\epsilon,i}|\leq \left( 9 n_x R_t^2 /\epsilon \right)^{n_x/2}$ for all $\epsilon\in(0,9 n_x R_t^2]$. Note that since $\mathcal{K}_{\epsilon,i}$ is an $\sqrt{\epsilon}$-net of set $\mathcal{K}_i$, by definition, for any $x\in X$, there exists $\hat{x}_i\in X_{\epsilon,i}$ such that $\|\nu_i(x)-\nu_i(\hat{x}_i)\|\leq\sqrt{\epsilon}$. Hence,
\begin{align*}
    &\|x-\hat{x}_i\|^2+t^2\|\nabla\phi_i(x)-\nabla\phi_i(\hat{x}_i)\|^2\\
    &\leq \|x-\hat{x}_i\|^2+2t(x-\hat{x}_i)^\top (\nabla\phi_i(x)-\nabla\phi_i(\hat{x}_i)) +t^2\|\nabla\phi_i(x)-\nabla\phi_i(\hat{x}_i)\|^2\\
    &=\|\nu_i(x)-\nu_i(\hat{x}_i)\|^2\\
    &\leq{\epsilon},
\end{align*}
where the first inequality is due to the convexity of $\phi_i$. This implies that for any $x\in X$, there exists $\hat{x}_i\in X_{\epsilon,i}$ such that $\|x-\hat{x}_i\|$ is controlled by $\sqrt{\epsilon}$ and $\|\nabla\phi_i(x)-\nabla\phi_i(\hat{x}_i)\|$ is bounded by $\sqrt{\epsilon}/t$
.

Now, consider $K \coloneqq \left( 18 n_x R_t^2 /\epsilon \right)^{n_x/2}$ and  set ${X}_{K,i}\triangleq\left\{ \hat x_{1}^{(i)}, \dots, \hat x_{K}^{(i)} \right\} \subseteq {X}$ such that ${X}_{\epsilon/2,i} \subseteq{X}_{K,i}$. 
Then, we introduce the following piecewise affine function $h:X\to \mathbb{R}$
\begin{align*}
    h(x) =& \max_{k = 1, \dots, K}\left\{ \phi_1(\hat x_k^{(1)}) + \nabla \phi_1(\hat x_k^{(1)})^\top (x-\hat x_k^{(1)}) \right\} \\&\quad- \max_{k = 1, \dots, K}\left\{ \phi_2(\hat x_k^{(2)}) + \nabla \phi_2(\hat x_k^{(2)})^\top (x-\hat x_k^{(2)}) \right\}.
\end{align*}
Hence, for any $x\in X,$ we have that
\begin{align*}
  |f(x) -  h(x)| &\leq  \left|\phi_1(x)-\max_{k = 1, \dots, K}\left\{ \phi_1(\hat x_k^{(1)}) + \nabla \phi_1(\hat x_k^{(1)})^\top (x-\hat x_k^{(1)}) \right\}\right| \\&\quad+\left|\phi_2(x) - \max_{k = 1, \dots, K}\left\{ \phi_2(\hat x_k^{(2)}) + \nabla \phi_2(\hat x_k^{(2)})^\top (x-\hat x_k^{(2)}) \right\}\right|\\
  &=\sum_{i=1,2}\phi_i(x)-\max_{k = 1, \dots, K}\left\{ \phi_i(\hat x_k^{(i)}) + \nabla \phi_i(\hat x_k^{(i)})^\top (x-\hat x_k^{(i)}) \right\},
\end{align*}
where the last equality is because the function $\max_{k = 1, \dots, K}\left\{ \phi_i(\hat x_k^{(i)}) + \nabla \phi_i(\hat x_k^{(i)})^\top (x-\hat x_k^{(i)})\right\}$ is a uniform lower bound of $\phi_i$ by convexity. Let us define the selective function
\begin{align*}\label{eqn:val_k_i}
    k_i(x) = \underset{k= 1, \dots, K}{\text{argmin}}\left\| \nu_i^{-1}(x) - \nu_i^{-1}(\hat x_k^{(i)}) \right\|,
\end{align*}
which selects the index of the point in $X_{K,i}$ such that $x_k^{(i)}$ is closest to $x$ as measured by $\nu_i^{-1}$. Since $X_{K,i}$ is the preimage of $\mathcal{K}_{\epsilon/2,i}$, which is, by definition, an $\epsilon/2$-cover of  $\mathcal{K}_{i}$, 
\begin{align*}
  |f(x) -  h(x)|&\leq \sum_{i=1,2}\phi_i(x)- \phi_i(\hat x_{k_i(x)}^{(i)}) + \nabla \phi_i(\hat x_{k_i(x)}^{(i)})^\top (x-\hat x_{k_i(x)}^{(i)}) \\
  &\leq \sum_{i=1,2} {\|\nabla  \phi_i(x) - \nabla  \phi_i(\hat x_{k_i}^{(i)})\|} {\|x-\hat x_{k_i}^{(i)}\|}\\
  &\leq \frac{\epsilon}{t}, 
\end{align*}
where the first inequality is by plugging $k_i(x)$ into the maximum operator, the second inequality is due to Cauchy-Schwarz, and the last inequality follows from the fact that  $\|x-\hat{x}_i\|$ is controlled by $\sqrt{\epsilon/2}$ and $\|\nabla\phi_i(x)-\nabla\phi_i(\hat{x}_i)\|$ is bounded by $\sqrt{\epsilon/2}/t$
by the aforementioned reasoning. Therefore, we have shown that $h$ can uniformly approximate $f$ by accuracy $\epsilon/t$. 

From $K \coloneqq \left( 18 n_x R_t^2 /\epsilon \right)^{n_x/2}$, we have $\epsilon = 18 n_x R_t^2K^{-2/n_x}$. Therefore, $$\|f - h\|_\infty\leq \frac{\epsilon}{t} = \frac{18n_xR_t^2}{t}K^{-2/n_x}.$$
Optimizing over $t$ optimal, we obtain $t^* = \frac{1}{2L}$. Therefore, by choosing  $K^* = \left(\frac{\epsilon}{144 n_x L }\right)^{\frac{-n_x}{2}}$, we have that $\|f - h\|_\infty\leq \epsilon$. The proof is concluded by recalling Lemma \ref{lem:max_affine_reconstruct}.

\end{proof}

\section{Proof in Section ``the role of depth''}\label{app:role_depth}

Before we prove the main theorem, we first establish an intermediate result regarding the exact reconstruction of polynomial functions.

\begin{lemma}\label{lem:reconstruct_poly_multilayer}
Let $P_\mathbf{m}(x)\coloneqq(x-\frac{2\mathbf{m}+1}{2N})^{\mathbf{n}}=\prod_{i=1}^{n_x} (x_i-\frac{2m_i+1}{2N})^{n_i}$ be a polynomial of order $k\geq 2$, where  $\mathbf{m}\coloneqq (m_1,...,m_{n_x})\in\{0,1,...,N-1\}^{n_x}$ and $\mathbf{n}\in\{0,1,...,k\}^{n_x}$ such that $|\mathbf{n}|=\sum_i n_i\leq k$. Suppose that $x_i\neq \frac{2m_i+1}{2N}$ for all $i\in[n_x]$. Then, we can reconstruct $P_\mathbf{m}(x)$ by  a chain graph with depth of at most $2(k-1)$, where each node is an LP with at most 1 variable and 2 constraints.
\end{lemma}
\begin{proof}
 The proof is inspired by the observation that we can exactly reconstruct a product function by a composition of LPs. In particular, we deal with the case that the input can be negative or translated with some minor modifications to the reconstruction. First, the solution to $\{\min_{z} -z \;\;\text{s.t.} \;\;(x_1-a_1)z\leq x_2-a_2,\;\;(x_1-a_1)z\geq x_2-a_2 \}$, where $x_1\neq a_1$ is exactly $(x_2-a_2)/(x_1-a_1)$. Second, a two-layer architecture, with $\{\min_{z} -z \;\;\text{s.t.} \;\;(x_1-a_1)z\leq 1,\;\;(x_1-a_1)z\geq 1 \}$ as the first layer, the output of which is provided as input to $\{\min_{z} -z \;\;\text{s.t.} \;\;\square z\leq x_2-a_2,\;\;\square z\geq x_2-a_2 \}$ as the value for parameter $\square$, has an output that is $(x_1-a_1)(x_2-a_2)$ for any $x_1\neq a_1$. Thus, the result follows by repeated composition of product functions.
\end{proof}

\subsection{Proof of Theorem \ref{thm:multilayer_n}}\label{subsec:multilayer_n}

\begin{proof}
The proof follows the idea of \cite[Thm. 1]{yarotsky2017error} that performs local Taylor approximation for each cell of a partition of controlled resolution. We will focus mainly on the deviation points.

Consider a partition of the space $[0,1]^{n_x}$ into $N^{n_x} $ cells of equal size. Denote 
$$I_\mathbf{m}(x) = \begin{cases}
  1 \ \text{if } \left| x_i - \frac{2m_{i}+1}{2N} \right|\leq \frac{1}{2N},\ \forall i=1,...,n_x\\ 0 \ \text{ otherwise } 
\end{cases}$$
as the indicator function of the cell indexed by $\mathbf{m}$, where $\mathbf{m} = (m_1,\dots, m_{n_x})\in \{0,1, \dots, N-1\}^{n_x}$. Also, let 
$$h_\mathbf{m}(x) = \sum_{\mathbf{n}:|\mathbf{n}|\leq k}\frac{D^{\mathbf{n}} f}{\mathbf{n}!}\bigg\rvert_{x=\frac{2\mathbf{m}+1}{2N}}\left(x-\frac{2\mathbf{m}+1}{2N}\right)^{\mathbf{n}}$$ 
be the Taylor approximation of $f$ within the cell indexed by $\mathbf{m}$ or order $\mathbf{n}$, where the point of approximation is selected at the center of the cell. Here, $\mathbf{n}!=\prod_{i=1}^{n_x} n_i!$ as usual. Let $\hat{f}(x)=\sum_{\mathbf{m}}I_\mathbf{m}(x)h_\mathbf{m}(x)$, which pieces together local approximations for each cell of the partition. Then the approximation error can be bounded by the standard argument for Taylor's expansion:
\begin{align}
    \left|f(x) - \hat f(x)\right| &= \left|   \sum_m I_\mathbf{m}(x)\left(f(x) - h_\mathbf{m}(x)\right) \right|\\
    & \leq \underset{\mathbf{m}:\left| x_i - \frac{{m}_{i+1}}{2N} \right|< \frac{1}{2N}\ \forall i}{\text{max } } \left| f(x) -h_\mathbf{m}(x) \right|\\
    & \leq \frac{n_x^k}{k!}\left(\frac{1}{N}\right)^k \underset{\mathbf{n}:|\mathbf{n}|=k}{\text{max }} \underset{x\in [0,1]^{n_x}}{\text{ess sup}} \left| D^\mathbf{n} f (x) \right|\label{eq:improve_yaroski}\\
    &  \leq \frac{n_x^k}{k!}\left(\frac{1}{N}\right)^k, 
\end{align}
where the first equality is due to $f(x)= \sum_m I_\mathbf{m}(x)f(x)$, the first inequality is because there is no overlapping among the supports of $I_\mathbf{m}$, the second inequality is a standard Taylor approximation bound,  the last inequality is due to the model class assumption. Substituting $N = \lceil n_x (\frac{1}{k!\epsilon})^{1/k} \rceil$ gives the error of $\epsilon$. Note that the second step \eqref{eq:improve_yaroski}, a reduction at the order of $2^{n_x/k}$ compared to \cite[Thm. 1]{yarotsky2017error} is achieved due to the exact partitioning of the space with the bump function.

What is left of the proof is to reconstruct the function $\sum_\mathbf{m} I_\mathbf{m}(x)h_\mathbf{m}(x)$. To this end, we first recognize that 
\begin{equation*}
\begin{aligned}
I_\mathbf{m}(x) =\underset{z\in[0,1]}{\arg\min } \quad &- z\\
\textrm{s.t.} \quad & \left( x_i - \frac{2m_i+1}{2N} - \frac{1}{2N} \right)z \leq 0 \quad \forall i=1,...,n_x\\
  &\left( x_i - \frac{2m_i+1}{2N} + \frac{1}{2N} \right)z \geq 0 \quad \forall i=1,...,n_x ,
\end{aligned}
\end{equation*}
which is a bump function in $n_x$-dimensional space. Notice that the above optimization has $2n_x +2$ constraints and 1 variable. 

Also, note that $h_\mathbf{m}(x)$ can be reconstructed by first reconstructing each polynomial separately, which requires a chain with at most $2(k-1)$ depth, where each node is an LP with at most 1 variable and 2 constraints. Then, we take a weighted sum over all polynomials, which is simply an affine transformation on the outputs of each chain graph, so no additional layers are needed. Thus, in total, we can have a graph with a depth of at most $2(k-1)$, where each layer has at most $\begin{pmatrix}n_x\\k+n_x\end{pmatrix}$ LPs, each with at most 1 variable and 2 constraints.

Finally, to reconstruct $\hat f(x)$, we can first multiply the outputs of $I_\mathbf{m}(x)$ and $h_\mathbf{m}(x)$, which requires an additional 2 LP layers. Also note that the width of the first layer is also increased by $2n_x+2$ constraints and 1 variable due to the need for $I_\mathbf{m}$. Since there are in total $N^{n_x}$ possible choices of $\mathbf{m}$, to implement the product, we need $N^{n_x}$ 2-layered LPs, each with 3 constraints and 1 variable, whose outputs are then summed up in the last layer.

Summing up, we have constructed a network with a depth of $2k$, where the first layer has $\bigg( 2n_x+2+2\begin{pmatrix}n_x\\k+n_x\end{pmatrix}\bigg)N^{n_x}$ constraints and $\bigg(1+\begin{pmatrix}n_x\\k+n_x\end{pmatrix}\bigg)N^{n_x}$ variables, and the last layer has $3N^{n_x}$ constraints and $N^{n_x}$ variables. Observe that the widest layer is the first layer. The result follows by plugging in $N = \lceil n_x (\frac{1}{k!\epsilon})^{1/k} \rceil$ in the above.

\end{proof}

\begin{remark}
By comparing with the order given in \cite[Thm. 1]{yarotsky2018optimal} for the case of DNN, we can see that our construction has a fixed depth that does not grow with the accuracy requirement $\epsilon$, and also a width that is reduced by the order of $2^{n_x^2/k}$, which can be substantial for high-dimensional problems.
\end{remark}

\begin{remark}
To stress on the exact reconstruction, we provide an example of how  we perform  efficient reconstruction of $\underbrace{\prod_{k=1}^2\psi(4Nx_k)}_{\phi_0(\mathbf{x})}x_1^3x_2$, visualized in Figure \ref{fig:talor_polynomial}. 
\begin{figure}[]
    \centering
    \includegraphics[width=0.8\linewidth]{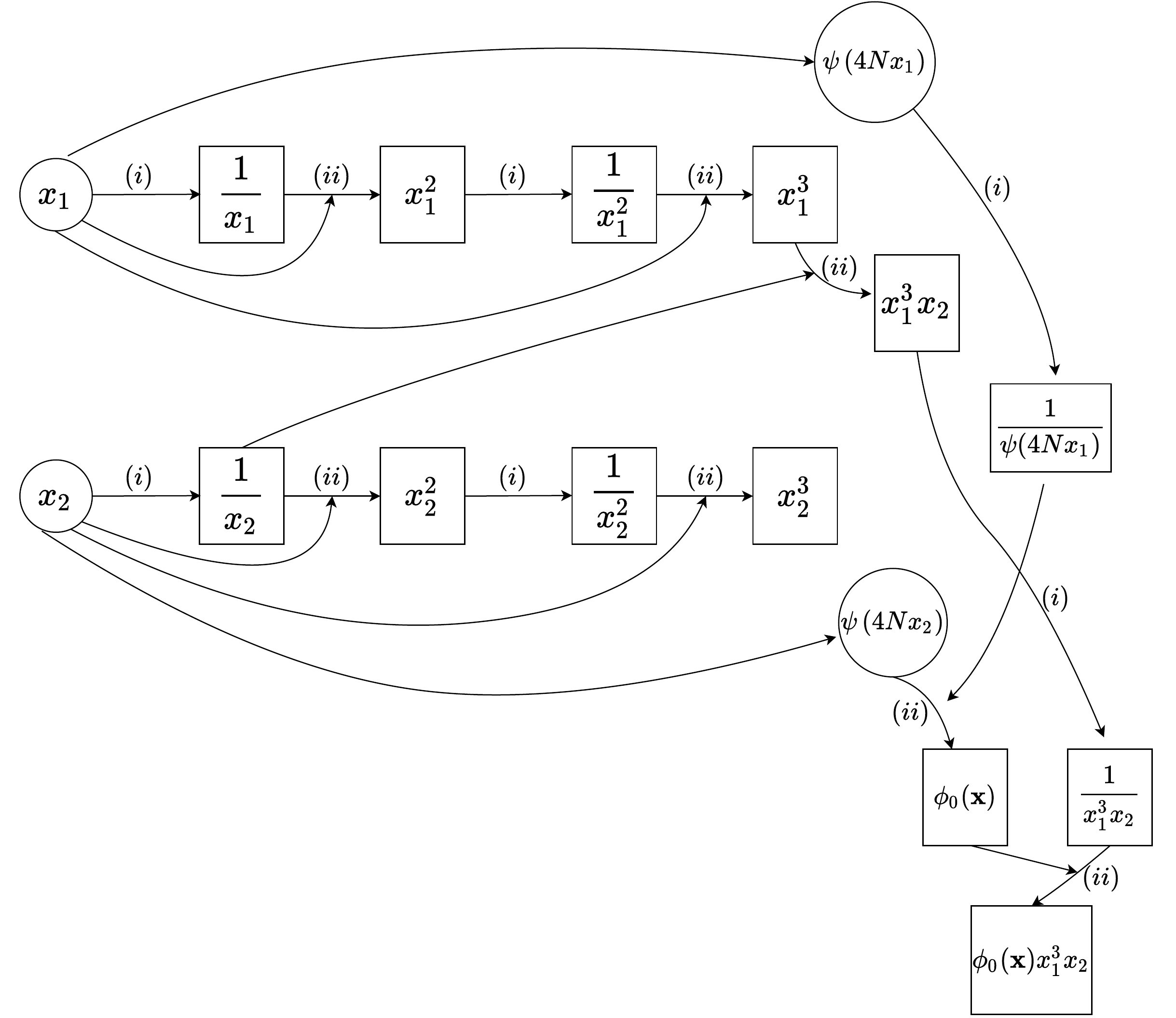}    \caption{\label{fig:talor_polynomial}\footnotesize Reconstruction of Taylors' polynomial, where we use $``(i)"$ to denote an optimization to obtain the inverse of  the input, further $``(ii)"$ denotes an optimization to obtain the product of two inputs. }
\end{figure}

\end{remark}

\section{Proof in Section ``Definability and whitney stratification''} \label{app:definability}

For the sake of completeness, let us recall some fundamental concepts/results in tame geometry, which allow us to study the global geometry of the solution function of a wide range of optimization problems. More information can be found in \citep{van1996geometric,ioffe2009invitation,davis2020stochastic}.

\begin{definition}[Whitney Stratification]
\label{def:WhitneyStratification}
 A Whitney $C^k$ stratification of a set $I$ is a partition of $I$ into finitely many nonempty $C^k$ manifolds, called strata, satisfying the following compatibility conditions:

\begin{enumerate}
    \item For any two strata $I_a$ and $I_b$, the implication $I_a \cap I_b \neq \emptyset$ implies that $I_a \subset \mathrm{cl} I_b$ holds, where $\mathrm{cl} I_b$ denotes the closure of the set $I_b$. 
    
    \item For any sequence of points $x_k$ in a stratum $I_a$, converging to a point $x^{\star}$ in a stratum $I_b$, if the corresponding normal vectors $v_k \in N_{I_a}(x_k)$ converge to a vector $v$, then the inclusion $v \in N_{I_b}(x^{\star})$ holds. Here $N_{I_a}(x_k)$ denotes the normal cone to $I_a$ at $x_k$.
\end{enumerate}

\end{definition}

Roughly speaking, stratification is a locally finite partition of a given set into differentiable manifolds, which fit together in a regular manner (property $1$ in Def. \ref{def:WhitneyStratification}). Whitney stratification as defined above is a special type of stratification for which the strata are such that their tangent spaces (as viewed from normal cones) also fit regularly (property $2$).

There are several ways to verify Whitney stratifiability. For example, one can show that the function under study belongs to one of the well-known function classes, such as semialgebraic functions \citep{van1996geometric}, whose members are known to be Whitney stratifiable. However, to study the solution function of a general convex optimization problem, we need a far-reaching axiomatic extension of semialgebraic sets to classes of functions definable on ``o-minimal structures,'' which are very general classes and share several attractive analytic features as semialgebraic sets, including Whitney stratifiability \citep{davis2020stochastic,van1996geometric}.

\begin{definition}[o-minimal structure]
\label{def:0-minimal structure}
 \citep{van1996geometric} An o-minimal structure is defined as a sequence of Boolean algebras $O_v$ of subsets of $\mathbb{R}^{v}$, such that for each $n_v \in \mathbb{N}$, the following properties hold:

\begin{enumerate}
    \item If some set $X$ belongs to $O_v$, then $X \times \mathbb{R}$ belong to $O_{v+1}$.
    
    \item Let $P_{proj}: \mathbb{R}^{v} \times \mathbb{R} \rightarrow \mathbb{R}^{v}$ denote the coordinate projection operator onto $\mathbb{R}^{v}$, then for any $X$ in $O_{v+1}$, the set $P_{proj}(X)$ belongs to $O_v$.
    
    \item $O_v$ contains all sets of the form $\{x \in \mathbb{R}^{v}: \hspace{0.1cm} y(x) = 0 \}$, where $y(x)$ is a polynomial in $\mathbb{R}^{v}$.
    
    \item The elements of $O_1$ are exactly the finite unions of intervals (possibly infinite) and points.
\end{enumerate}
Then all the sets that belong to $O_v$ are called definable in the o-minimal structure.

\end{definition}
Definable sets have broader applicability than semialgebraic sets (in the sense that the latter is a special kind of definable sets) but enjoy the same remarkable stability property: the composition of definable mappings (including sum, inf-convolution, and several other classical operations of analysis involving a finite number of definable objects) in some o-minimal structure remains in the same structure.

\subsection{Proof of Theorem \ref{prop:definability_pi}}

\begin{proof}
Let $g_R(z;x,\theta)\coloneqq g(z;x,\theta)+\mathbb{I}_{R(x,\theta)}$ be the penalized formulation of the optimization problem in \ref{prob:main_prob}. By Assumption \ref{asmptn:Definable}, the definability of the indicator function, and the fact that definability is preserved under addition and composition, which is due to the definable counterpart of the Tarski-Seidenberg theorem \cite{van1996geometric}, $g_R(z;x,\theta)$ is definable on the same o-minimal structure.

Let $g^*(x,\theta) \coloneqq \min_z\left\{ g(z; x, \theta) \ |\ \text{subject to } z\in R(x, \theta)\right\}=\min_z g_R(z;x,\theta)$ be the optimal value function. Since definability is preserved under $\inf$ projection, $g^*(x,\theta)$ is definable. Also, recognize that $\pi(x,\theta) = \{z:g_R(z;x,\theta) = g^*(x,\theta)\}$, by definition, $\pi(x,\theta)$ is definable on the same o-minimal structure. Since the output of any function in $\Upsilon^{W^v,W^c,L}$ is given by finite operations of affine transformation and the composition of definable functions (as stipulated by network construction), it is definable by the same reasoning as above.
\end{proof}

\section{Proof in Section ``Covering number bounds''}\label{app:statistical_bnds}

We start with some basic definitions (see, e.g., \cite{van1996weak,mohri2018foundations,wainwright2019high}).

\begin{definition}[Covering number and entropy]
\label{def:coveringNumbers} For a given metric  $d$ and $\epsilon >0$, the covering number $\mathcal{N}(\epsilon,\mathcal{F},d)$ is the minimal number of balls $\mathbb{B}_{\epsilon}(f)\coloneqq \{g\in\mathcal{F}:d(g,f)\leq\epsilon\}$ of radius $\epsilon$ and center $f\in\mathcal{F}$ needed to cover
the set $\mathcal{F}$, i.e., $\mathcal{N}(\epsilon,\mathcal{F},d) = \min \{n \mid \mathcal{F} \subseteq \cup_{i=1}^n \mathbb{B}_{\epsilon}(f_i), \;\; \text{for some} \; f_i \in \mathcal{F}\}.$  The entropy is the logarithm of the covering number.
\end{definition}
 
It is often the case that tighter bounds are possible by measuring complexity in a data-dependent manner. We can define the empirical versions of the above notions by defining the empirical metric with respect to a dataset $\mathcal{D}_n\coloneqq\{x_1,...,x_n\}$. In particular, we define the set $\mathcal{F}(\mathcal{D}_n) \in\mathbb{R}^n$ as follows:
$$\mathcal{F}(\mathcal{D}_n)\coloneqq\{(f(x_1),...,f(x_n))|f\in\mathcal{F}\},$$
along with a distance measured in terms of the empirical $\ell_1$-norm: $$\|f-g\|_{\mathcal{D}_n}\coloneqq\frac{1}{n}\sum_{i=1}^n|f(x_i)-g(x_i)|.$$
Using the above empirical $\ell_1$-norm as the metric and taking the supremum over all possible datasets $\mathcal{D}_n$ will lead to the definitions of empirical $L_1$ covering number. For simplicity of notation, we will use $\mathcal{N}_1(\epsilon,\mathcal{F},n)$, where $n$ is the number of data points.

\subsection{Proof of Theorem \ref{thm:coveringPiecewiseAffine}}\label{subsec:coveringPiecewiseAffine}

For convenience, we recall the setup here. Consider the function 
\begin{equation*}
    \pi_{QP}(x,\theta) \coloneqq {\arg\min}_{z\in R(x,\theta)} \left(\frac{1}{2}A_0 z+U_0^x x+U_0^\theta \theta +b_0\right)^\top z,
\end{equation*}
with $$R(x,\theta)\coloneqq\{z: A_{1}z\leq b_{1}+U_{1}^xx+U_{1}^\theta\theta,A_{2}z = b_{2} + U^x_{2} x+U^\theta_{2} \theta\},
$$ where $x$ and $\theta$ are the input and parameter, respectively, $z\in\mathbb{R}^{n_z}$ is the optimization variable, and all the rest are fixed hyperparameters of compatible dimensions; in particular, $A_0\succ 0$ is positive definite, and let $m_{1}$ and $m_{2}$ be the number of inequality and equality constraints, respectively. We can also define $\pi_{LP}(x,\theta)$ by setting $A_0= 0$.  Define $\Pi_\square=\{\pi_\square(\cdot,\theta):\theta\in\Theta\}$, where $\square$ can be LP or QP, and $\Theta$ is compact. 
    
Next, we develop some elementary results for our settings. Note that similar derivations can be also found in the multi-parametric literature (see, e.g., \cite{grancharova2012explicit,bemporad2015explicit,pistikopoulos2020multi}). We restate the lemmas for easy reference.
\begin{lemma}\label{lem:mp-qp}
    Consider QP \eqref{equ:qp} with $A_0\succ 0$. Denote $\mathcal{I}^*(x,\theta)$ be the set of active inequality constraints for any given pair of $(x,\theta)$ such that the corresponding inequalities hold with equality: $$[A_1]_{\mathcal{I}^*(x,\theta)}\pi_{QP}(x,\theta)=[b_1]_{\mathcal{I}^*(x,\theta)}+[U_1^x]_{\mathcal{I}^*(x,\theta)} x+[U_1^\theta]_{\mathcal{I}^*(x,\theta)}\theta$$ 
    Let
    \begin{equation*}
        \tilde{A}_{QP}(\mathcal{I}^*)=A_0^{-1}[A_1]_{\mathcal{I}^*}^\top\left([A_1]_{\mathcal{I}^*}A_0^{-1}[A_1]_{\mathcal{I}^*}^\top\right)^{-1}\left([A_1]_{\mathcal{I}^*}A_0^{-1}U_0^x+[U_1^x]_{\mathcal{I}^*}\right)-A_0^{-1}[U_0^x]_{\mathcal{I}^*},
    \end{equation*}
    where we used the shorthand $\mathcal{I}^*$ for $\mathcal{I}^*(x,\theta)$. Then, the solution function can be written in the form:
    \begin{equation*}
        \pi_{QP}(x,\theta)=\tilde{A}_{QP}(\mathcal{I}^*)x+\tilde{b}_{QP}(\theta,\mathcal{I}^*),
    \end{equation*}
    for some (herein unspecified) bias function $\tilde{b}_{QP}(\theta,\mathcal{I}^*)$ that no longer  depends on $x$ if $\mathcal{I}^*(x,\theta)$ is given.
\end{lemma}
\begin{proof}
By complementarity, the set of inactive inequality constraints is given as $\mathcal{N}^*(x,\theta)\coloneqq\{1,...,m_1\}\setminus\mathcal{I}^*(x,\theta)$. In the following, we omit the notational dependence of $\mathcal{I}^*(x,\theta)$ or $\mathcal{N}^*(x,\theta)$ on $(x,\theta)$, and denote $z^*$ as the optimal solution $\pi_{QP}(x,\theta)$ for simplicity and $\lambda^*$ as the optimal dual variables. The optimal solution $z^*$ for a fixed $(x,\theta)$ is fully characterized by the Karush-Kuhn-Tucker (KKT) conditions:
\begin{subequations}
\begin{align}
    &A_0z^*+\underbrace{U_0^x x+U_0^\theta \theta+b_0}_{\tilde{b}_0(x,\theta)}+[A_1]_{\mathcal{I}^*}^\top\lambda^*=0\\
    &[A_1]_{\mathcal{I}^*}z^*=\underbrace{[b_1]_{\mathcal{I}^*}+[U_1^x]_{\mathcal{I}^*} x+[U_1^\theta]_{\mathcal{I}^*}\theta}_{\tilde{b}_1(x,\theta,\mathcal{I}^*)}\label{eq:qp-cr-1}\\
     &[A_1]_{\mathcal{N}^*}z^*<[b_1]_{\mathcal{N}^*}+[U_1^x]_{\mathcal{N}^*} x+[U_1^\theta]_{\mathcal{N}^*}\theta\label{eq:qp-cr-2}\\
     &\lambda^{*\top}([A_1]_{\mathcal{I}^*}z^*-[b_1]_{\mathcal{I}^*}-[U_1^x]_{\mathcal{I}^*} x-[U_1^\theta]_{\mathcal{I}^*}\theta)=0\\
      &\lambda^*\geq 0
\end{align}
\end{subequations}
If $[A_1]_{\mathcal{I}^*}$ has full row rank (which can be satisfied by some standard constraint qualifications \cite{pistikopoulos2020multi}), we have that
\begin{equation*}
    \lambda^*=-\left([A_1]_{\mathcal{I}^*}A_0^{-1}[A_1]_{\mathcal{I}^*}^\top\right)^{-1}\left([A_1]_{\mathcal{I}^*}A_0^{-1}\tilde{b}_0(x,\theta)+\tilde{b}_1(x,\theta,\mathcal{I}^*)\right)
\end{equation*}
and consequently,
\begin{equation*}
    z^*=A_0^{-1}[A_1]_{\mathcal{I}^*}^\top\left([A_1]_{\mathcal{I}^*}A_0^{-1}[A_1]_{\mathcal{I}^*}^\top\right)^{-1}\left([A_1]_{\mathcal{I}^*}A_0^{-1}\tilde{b}_0(x,\theta)+\tilde{b}_1(x,\theta,\mathcal{I}^*)\right)-A_0^{-1}\tilde{b}_0(x,\theta).
\end{equation*}
Hence, the conclusion follows by grouping terms by whether they depend on $x$ conditioning on $\mathcal{I}^*(x,\theta)$. In other words, given the set of active constraints $\mathcal{I}^*(x,\theta)$, the optimal solution is an affine function within the region where such active constraints hold (specified by \eqref{eq:qp-cr-1} and \eqref{eq:qp-cr-2}).
\end{proof}
    
\begin{lemma}\label{lem:mp-lp-app}
    Consider LP \eqref{equ:qp} with $A_0= 0$. Denote $\mathcal{I}^*(x,\theta)$ be the set of active inequality constraints for any given pair of $(x,\theta)$ such that the corresponding inequalities hold with equality: $$[A_1]_{\mathcal{I}^*(x,\theta)}\pi_{LP}(x,\theta)=[b_1]_{\mathcal{I}^*(x,\theta)}+[U_1^x]_{\mathcal{I}^*(x,\theta)} x+[U_1^\theta]_{\mathcal{I}^*(x,\theta)}\theta$$ 
    Let 
    \begin{equation*}
        \tilde{A}_{LP}(\mathcal{I}^*)=\begin{bmatrix}
    [A_{1}]_{\mathcal{I}^*}\\
    A_{2}
    \end{bmatrix}^{-1}\begin{bmatrix}
    [U_{1}^x]_{\mathcal{I}^*}\\
    U_{2}^x
    \end{bmatrix},
    \end{equation*}
    where we used the shorthand $\mathcal{I}^*$ for $\mathcal{I}^*(x,\theta)$. Then, the solution function can be written in the form:
    \begin{equation*}
        \pi_{LP}(x,\theta)=\tilde{A}_{LP}(\mathcal{I}^*)x+\tilde{b}_{LP}(\theta,\mathcal{I}^*),
    \end{equation*}
    for some (herein unspecified) bias function $\tilde{b}_{LP}(\theta,\mathcal{I}^*)$ that does not  depend on $x$ given $\mathcal{I}^*(x,\theta)$.
\end{lemma}

\begin{proof}
We follow the proof of Lemma \ref{lem:mp-qp} and only focus on the deviation points. With the same notations set up, the optimal solution $z^*$ for a fixed $(x,\theta)$ is fully characterized by the KKT conditions:
\begin{subequations}
\begin{align}
    &{U_0^x x+U_0^\theta \theta+b_0}+[A_1]_{\mathcal{I}^*}^\top\lambda^*=0\\
    &[A_1]_{\mathcal{I}^*}z^*={[b_1]_{\mathcal{I}^*}+[U_1^x]_{\mathcal{I}^*} x+[U_1^\theta]_{\mathcal{I}^*}\theta}\label{eq:lp-cr-1}\\
    &A_2 z^*={b_2+U_2^x x+U_1^\theta\theta}\label{eq:lp-cr-2}\\
     &[A_1]_{\mathcal{N}^*}z^*<[b_1]_{\mathcal{N}^*}+[U_1^x]_{\mathcal{N}^*} x+[U_1^\theta]_{\mathcal{N}^*}\theta\label{eq:lp-cr-3}\\
     &\lambda^{*\top}([A_1]_{\mathcal{I}^*}z^*-[b_1]_{\mathcal{I}^*}-[U_1^x]_{\mathcal{I}^*} x-[U_1^\theta]_{\mathcal{I}^*}\theta)=0\\
      &\lambda^*\geq 0
\end{align}
\end{subequations}
If $[[A_1]_{\mathcal{I}^*}^\top\; A_2^\top]^\top$ has full rank (which can be satisfied by some standard constraint qualifications, e.g., LICQ \cite{pistikopoulos2020multi}), we have that
\begin{equation*}
    z^*=\begin{bmatrix}
    [A_{1}]_{\mathcal{I}^*}\\
    A_{2}
    \end{bmatrix}^{-1}\begin{bmatrix}
    [U_{1}^x]_{\mathcal{I}^*}\\
    U_{2}^x
    \end{bmatrix}x+\begin{bmatrix}
    [A_{1}]_{\mathcal{I}^*}\\
    A_{2}
    \end{bmatrix}^{-1}\left(\begin{bmatrix}
    [U_{1}^\theta]_{\mathcal{I}^*}\\
    U_{2}^\theta
    \end{bmatrix}\theta+\begin{bmatrix}
    [b_{1}^\theta]_{\mathcal{I}^*}\\
    b_{2}^\theta
    \end{bmatrix}\right)
\end{equation*}
Hence, the conclusion follows by grouping terms by whether they depend on $x$ conditioning on $\mathcal{I}^*(x,\theta)$. In other words, given the set of active constraints $\mathcal{I}^*(x,\theta)$, the optimal solution is an affine function within the region where such active constraints hold (specified by \eqref{eq:lp-cr-1}-\eqref{eq:lp-cr-3}).
\end{proof}

Let us also recall that $\kappa_{LP}^\star=\max_{\iota\subseteq\{1,...,m_1\},\lvert\iota \rvert \leq n_z-m_2}\left\|\tilde{A}_{LP}(\iota)\right\|_2,$ and $\kappa_{QP}^\star=\max_{\iota\subseteq\{1,...,m_1\},\lvert\iota \rvert \leq n_z-m_2}\|\tilde{A}_{QP}(\iota)\|_2,$ where $\|\cdot\|_2$ is the spectral norm, and $\tilde{A}_{QP}(\iota)$ and $\tilde{A}_{LP}(\iota)$ are specified in \eqref{eq:A-qp} and \eqref{eq:A-LP}, respectively by replacing $\mathcal{I}^*$ with $\iota$ (i.e., treating $\iota$ as a set of active inequality constraints).

\begin{theorem}
\label{thm:coveringPiecewiseAffine-app}
 The $ L_1$ covering number of $\Pi_\square$ over bounded input space is controlled by
\begin{equation*}
\begin{aligned}
\label{eq:LPcoveringBound-app}
   \log \mathcal{N}_1(\epsilon,\Pi_\square,n) \lesssim \frac{\kappa_\square^{\star 2}}{\epsilon^2} {\sum}_{0\leq i\leq { n_z-m_2}}\begin{pmatrix}m_1\\ i\end{pmatrix},
\end{aligned}
\end{equation*}
where $\square$ can be either LP or QP.
\end{theorem}
\begin{proof}
Due to  the known results in multi-parametric programming \cite{grancharova2012explicit,bemporad2015explicit}, the active set depends on the set of active inequalities. By linear algebra, the maximum number of active constraints cannot be greater than the dimension of the decision variable $n_x$. Let us consider the set $\mathcal{K}$, which contains all the sets of possible active constraints. Then, for each set $\iota \in \mathcal{K}$, we can uniquely define a region $\mathcal{CR}_\iota$ (a.k.a., critical region) in the space $X\times\Theta$, and  the restriction of the solution function to each region is an affine function \citep{pistikopoulos2020multi}:
\begin{equation}
    \begin{aligned}
    [\pi_{\square}(x,\theta)]_{\mathcal{CR}_\iota} =\tilde{A}_{\square}(x,\iota)x+\tilde{b}_{\square}(\theta,\iota),
    \end{aligned}
\end{equation}
where $\square$ can be LP or QP, with corresponding matrices defined in \eqref{eq:A-LP} and \eqref{eq:A-qp}, respectively, and $[\pi_\square(x,\theta)]_{\mathcal{CR}_\iota}$ denotes the restriction of the function $\pi_\square(x,\theta)$ to $\mathcal{CR}_\iota$.

To bound the class of mp-LP or mp-QP, our strategy is to bound the number of critical regions and combine it with a bound on the ${L}_1$ covering number among all the region. The number of critical regions for mp-LP and mp-QP can be bounded by:
 \begin{equation}
 \label{eq:maxQP}
  {\sum}_{0\leq i\leq n_z-m_2}\frac{m_1!}{(m_1-i)!(i)!},
 \end{equation}
which simply enumerates all the possible combinations of inequality constraints (from none up to $n_z-m_2$ of them). Next, by \citep[Cor. 9]{kakade2008complexity}, the covering number of the affine function within each critical region can be bounded (up to some constant) by $\frac{\kappa^{\star 2}_{\square}}{\epsilon^2}$, where $\square$ can be LP or QP (note that \citep[Cor. 9]{kakade2008complexity} is proved for an even stronger case of ${L}_2$ covering number, which provides an upper bound on the ${L}_1$ covering number). Combining these bounds, we get the overall bound.
\end{proof}

\subsection{Proof of Theorem \ref{thm:appcovering_bound}}\label{subsec:appcovering_bound}

Before we prove the main theorem, we will first provide an empirical $L_1$ covering number bounds for the $C^k$ smooth function class. We note that bounding the entropy numbers of the classes $C^k$ with respect to the supremum norm was were among the first results after the introduction of the concept of covering numbers (e.g., see the proof in  \cite[Theorem 2.7.1]{van1994bracketing}); however, we remark that we extend the proof to the case of empirical $L_1$ covering bound.

\begin{lemma}[Empirical $L_1$ covering number bound for $C^k$ smooth functions]
\label{lem:CoveringBoundSmooth-app}
Let $\mathcal{F}$ be a class of $C^k$ smooth functions defined over the domain region $X$, where $X\subset \mathbb{R}^{n_x}$ is a bounded, closed convex set. Then, the following bound holds:
\begin{equation}
    \label{eq:CoveringBoundSmooth_f-app}
    \log \mathcal{N}_1(\epsilon, \mathcal{F}, n) \lesssim n\bigg(\frac{1}{\epsilon} \bigg)^{1/k}+k^{n_x}\log\bigg(\frac{1}{\epsilon} \bigg).
\end{equation}
\end{lemma}

\begin{proof}
As the function class $\mathcal{F}$ is $C^k$ smooth, we can apply the standard argument of Taylor's theorem to any interior point $ x \in \mathrm{int}\mathcal{D}$. Let $\delta=\epsilon^{1/k}$, we first form a $\delta$-net for the $n$ points in $\mathcal{D}_n$; then we augment this set so that the new set, denoted by $\mathcal{D}_{n,\delta}=\{\tilde{x}_1,...,\tilde{x}_m\}$, has an additional ``star'' property, that there exists a point, say $\tilde{x}_1$ without loss of generality, such that for any point $\tilde{x}_j$, there exists a path $(\tilde{x}_{1},\tilde{x}_{j_1},\tilde{x}_{j_2},...,\tilde{x}_{j})$ of variable length, such that the distance between any two adjacent point is bounded by $\delta$. Note that this set is typically much less than the $\delta$-cover set of the entire space $\mathcal{D}$, especially when $n\ll 1/\delta^{n_x}$. In particular, we can construct such a set with $\mathcal{O}(n/\delta)$ points (up to the constant determined by the diameter of $\mathcal{D}_n$), by simply linking each point $x_n$ to the center of the star $\tilde{x}_1$ and discretizing the path into segments of length $\delta$. We also make sure that $\mathcal{D}_n$ is included in the set $\mathcal{D}_{n,\delta}$, the inclusion of which does not change the order of the size of the set.

Similar to the proof of Theorem \ref{thm:multilayer_n}, let $\mathbf{n} = (n_1,\dots, n_{n_x})\in \{0,1, \dots, k\}^{n_x}$ and $|\mathbf{n}|\leq k$. Also, let 
\begin{equation*}
    A_\mathbf{n}f = \Bigg(\bigg\lfloor\frac{D^\mathbf{n}f(\tilde{x}_1)}{\delta^{k-|\mathbf{n}|}}\bigg\rfloor, ..., \bigg\lfloor\frac{D^\mathbf{n}f(\tilde{x}_m)}{\delta^{k-|\mathbf{n}|}}\bigg\rfloor \Bigg)\in\mathbb{R}^m,
\end{equation*}
where $\lfloor\cdot\rfloor$ is the floor function, and recall that $D^\mathbf{n}\coloneqq\frac{\partial^{|\mathbf{n}|}}{\partial x_1^{n_1}\cdots\partial x_{n_x}^{n_{n_x}}}$ is the standard differential operator. Then, the vector $\delta^{k-|\mathbf{n}|}A_\mathbf{n}f$ consists of the values ${D^\mathbf{n}f(\tilde{x}_j)}$ discretized on a grid of mesh-width $\delta^{k-|\mathbf{n}|}$. 

If two function $f, g\in \mathcal{F}$ satisfy $A_\mathbf{n}f = A_\mathbf{n}g$ for each $\mathbf{n}$ with $|\mathbf{n}|\leq k$, then, by standard error bound of Taylor expansion, we have that
\begin{align*}
    \|f-g\|_{\mathcal{D}_n}&=\frac{1}{n}\sum_{i=1}^n|f(x_i)-g(x_i)|\\
    &\lesssim  \sup_{i\in[n]}\inf_{j\in [m]}\left|\sum_{\mathbf{n}:|\mathbf{n}|\leq k}\frac{D^{\mathbf{n}} (f-g)}{\mathbf{n}!}\bigg\rvert_{x=\tilde{x}_j}(x_i-\tilde{x}_j)^{\mathbf{n}}+\|x_i-\tilde{x}_j\|^k\right|\\
    &\lesssim \sum_{\mathbf{n}:|\mathbf{n}|\leq k}\frac{\delta^{k-|\mathbf{n}|}}{\mathbf{n}!}\delta^{\mathbf{n}}+\delta^k\\
    &\leq
    \delta^k(1+e^{n_x}),
\end{align*}
where the second inequality is by simply selecting the point $\tilde{x}_j=x_i$, which is possible since $\mathcal{D}_n\subset\mathcal{D}_{n,\delta}$. The constants omitted above only depend on the diameter of the set. Note that $\delta^k$ is the resolution observed at the zero-th order $A_\mathbf{n}f$ when $\mathbf{n}=0$.  
Here, $h^{\mathbf{n}}/\mathbf{n}!=\prod_{i=1}^{n_x} h_i^{n_i}/n_i!$ as usual. Informed by the above result, our strategy to bound the covering number $\mathcal{N}_1(\epsilon, \mathcal{F}, n)$ is based on bounding the number of different matrices
\begin{equation*}
    \label{eq:differentMatrices}
    Af = \begin{pmatrix}A_{0,0,...,0}f\\
                         A_{1,0,...,0}f\\
                         \vdots\\
                         A_{0,0,..., k}f\end{pmatrix},
\end{equation*}
where each row corresponds to $A_\mathbf{n}f$ for some $\mathbf{n}$ such that $|\mathbf{n}|\leq k$ and $f$ ranges over the class of $C^k$ smooth functions. 

By a simple combinatorial argument, the number of rows in $Af$ is less than $(k+1)^{n_x}$ for any $f\in C^k$. By the definition of $A_\mathbf{n}f$ and that $|D^\mathbf{n}f(\tilde{x}_j)|\leq 1$ for each $j\in[m]$, the number of possible values of each element in row $A_\mathbf{n}f$ is bounded by $2/\delta^{k-|\mathbf{n}|}+1$, which does not exceed $2\delta^{-k}+1$ since $\delta<1$. Thus, each column of the matrix can have at most $(2\delta^{-k}+1)^{(k+1)^{n_x}}$ different values. 

By our construction, for any $j\in[m]$ there is a path linking $\tilde{x}_{j}$ to $\tilde{x}_1$, where the distance between any two consecutive points is bounded by $\delta$. Therefore, we can organize the index in such a way that for each $j>1$, there is an index $i<j$ such that $\|\tilde{x}_i-\tilde{x}_{j}\|<\delta$. Then, use the crude bound previously obtained for the first column, and for each subsequent column, corresponding to $\tilde{x}_j$, there exists a point $\tilde{x}_i$ with $\|\tilde{x}_i-\tilde{x}_j\|\leq \delta$ and $i<j$. By Taylor's theorem,
\begin{equation*}
    D^\mathbf{n}f(\tilde{x}_j)=\sum_{|\mathbf{n}|+|\mathbf{n'}|\leq k}D^{\mathbf{n}+\mathbf{n}'}f(\tilde{x}_i)\frac{(\tilde{x}_i-\tilde{x}_j)^{\mathbf{n}'}}{\mathbf{n}'!}+R,
\end{equation*}
where $|R|\lesssim\|\tilde{x}_i-\tilde{x}_j\|^{k-|\mathbf{n}|}$. Thus, with $B_\mathbf{n}f=\delta^{k-|\mathbf{n}|}A_\mathbf{n}f$, we have that
\begin{align*}
    &\left|D^\mathbf{n}f(\tilde{x}_j)-\sum_{|\mathbf{n}|+|\mathbf{n'}|\leq k}B_{\mathbf{n}+\mathbf{n}'}f(\tilde{x}_i)\frac{(\tilde{x}_i-\tilde{x}_j)^{|\mathbf{n}'|}}{\mathbf{n}'!}\right|\\
    &\lesssim \sum_{|\mathbf{n}|+|\mathbf{n'}|\leq k}\left|B_{\mathbf{n}+\mathbf{n}'}f(\tilde{x}_i)-D^{\mathbf{n}+\mathbf{n}'}f(\tilde{x}_i)\right|\frac{(\tilde{x}_i-\tilde{x}_j)^{|\mathbf{n}'|}}{\mathbf{n}'!}+\delta^{k-|\mathbf{n}|}\\
    &\leq \sum_{|\mathbf{n}|+|\mathbf{n'}|\leq k}\delta^{k-|\mathbf{n}|-|\mathbf{n'}|}\frac{\delta^{|\mathbf{n}'|}}{\mathbf{n}'!}+\delta^{k-|\mathbf{n}|}\\
    &\lesssim \delta^{k-|\mathbf{n}|}.
\end{align*}
Thus, given the values in the $i$-th column of $Af$, the values $D^\mathbf{n}f(\tilde{x}_j)$ range over an interval of length proportional to $\delta^{k-|\mathbf{n}|}$. By normalizing with $\delta^{k-|\mathbf{n}|}$, it follows that the values in the $j$-th column of $Af$ range over integers in an interval of length proportional to $\delta^{k-|\mathbf{n}|}/(\delta^{k-|\mathbf{n}|})=1$. Thus, by a combinatorial argument, there exists a constant $C$ depending only on $k$ and $n_x$ such that the number of distinct matrices $Af$ is bounded by $(2\delta^{-k}+1)^{(k+1)^{n_x}}C^{m-1}$. The theorem follows by replacing $\delta$ by $\epsilon^{1/k}$ and $m$ by its upper bound $n/\delta=n\epsilon^{-1/k}$.

\end{proof}
We are now ready to prove the main result for a general optimization problem. As implied by the Whitney stratification of the solution map (Theorem \ref{prop:definability_pi}), there exists a finite partition of the domain, where the function restricted to each partition region is smooth. However, in general, we note that the partition region may be nonconvex.
\subsection{Proof of Theorem \ref{thm:appcovering_bound}}

\begin{proof}
We know from Theorem \ref{prop:definability_pi} that the solution mapping of a general optimization is $C^k$ smooth in each partition region. In addition, the number of partition regions $d_k$ is finite due to whitney stratification \cite{van1996geometric,ioffe2009invitation}. We will first bound the bracketing number \citep{van1994bracketing} for the function class $\mathcal{F}$, which can be used to bound the overall covering number of the function class $\Pi$. As we are dealing with a general nonlinear optimization problem, unlike the polytope partition for LP or QP, the partition regions might be nonconvex.

Let each partition region (possibly nonconvex) be denoted by $I_j$. To leverage the result from Lemma \ref{lem:CoveringBoundSmooth-app}, we form convex hulls for each partition region of the domain, denoted by $I_j'$. Note that it is possible and permitted to have overlaps between these convex hulls.

Create an $\epsilon $-net $\mathcal{F}_{j,\epsilon }=\{ f_{j,1},...,f_{j,p_j}\}$ for the set of $C^k$ functions defined on each convex hull $I_j'$ with respect to the empirical $L_1$ distance measured on $\mathcal{D}_n$. Note that we do not have any assumptions about the distribution of points in $\mathcal{D}_n$, thus we consider the case in the worst sense. Then, using Lemma \ref{lem:CoveringBoundSmooth-app}, $p_j$ can be selected to satisfy
\begin{equation}
    \log p_j \lesssim n\bigg(\frac{1}{\epsilon } \bigg)^{1/k}+k^{n_x}\log\bigg(\frac{1}{\epsilon } \bigg).
\end{equation}

Consider the set of functions $$\mathcal{F}_{\epsilon}\coloneqq\left\{f:f=\sum_{j=1}^{d_k}f_{j,i_j}\mathbb{I}(I_j),\;\;\forall\; i_j\in[p_j]\right\},$$
where each member function pieces together one of the $\epsilon $-set from every region. Hence, $|\mathcal{F}_{\epsilon}|=\prod_{j=1}^{d_k}p_j.$ 
Since $\mathcal{F}_{j,\epsilon }$ is an $\epsilon $-cover on the region of $I_j'\supseteq I_j$, there exists a selection function $i_j(\pi)$ such that for any $\pi\in\Pi$, $\sup_{\mathcal{D}'_n\in I_j^n}\|\pi(x)-f_{j,i_j(\pi)}(x)\|_{\mathcal{D}'_n}\leq \epsilon $. Let $\mathcal{D}_{n,j}=\{x:x\in I_j\cap \mathcal{D}_n\}$ be the subset of data that lie in $I_j$. To bound the empirical $L_1$ distance between any function $\pi\in\Pi$ to the set $\mathcal{F}_{\epsilon}$:
\begin{align*}
    &\min_{f\in \mathcal{F}_{\epsilon}}\|\pi-f\|_{\mathcal{D}_n}\\
    &=\min_{f\in \mathcal{F}_{\epsilon}}\frac{1}{n}\sum_{i=1}^n|\pi(x_i)-f(x_i)|\\
    &=\min_{f\in \mathcal{F}_{\epsilon}}\frac{1}{n}\sum_{j=1}^{d_k}\sum_{x_i\in \mathcal{D}_{n,j}}|\pi(x_i)-f(x_i)|\\
    &\leq \frac{1}{n}\sum_{j=1}^{d_k}\frac{|\mathcal{D}_{n,j}|}{|\mathcal{D}_{n,j}|}\sum_{x_i\in \mathcal{D}_{n,j}}|\pi(x_i)-f_{j,i_{j}(\pi)}(x_i)|\\
     &\leq \frac{1}{n}\sum_{j=1}^{d_k}{|\mathcal{D}_{n,j}|}\epsilon\\
    &\leq \epsilon,
\end{align*}
where the first inequality is by the selection of $f=\sum_{j=1}^{d_k}f_{j,i_j(\pi)}\mathbb{I}(I_j)\in\mathcal{F}_\epsilon$, the second inequality is because $\mathcal{N}_1(\epsilon, \mathcal{F}, n')\leq \mathcal{N}_1(\epsilon, \mathcal{F}, n)$ for any $n'\leq n$, and the last equality is due to $n=\sum_{j=1}^{d_k}{|\mathcal{D}_{n,j}|}$.
Thus, $\mathcal{F}_{\epsilon}$ forms an $\epsilon$-net of $\Pi$. The result follows by taking the logarithm of $|\mathcal{F}_{\epsilon}|$.

\end{proof}

\section{Numerical experiments}
In this section, we provide numerical examples to demonstrate the expressiveness of solution functions on two types of data: \emph{(1)} image processing, and \emph{(2)} reconstruction of SciPy test functions \footnote{\url{http://infinity77.net/go_2021/scipy_test_functions.html#scipy-test-functions-index}}. The authors  note that these are \emph{not} actual or intended applications of solution functions, but rather visual examples showing some complex functions that solution functions can represent. Throughout this section, we limit ourselves to the solution functions of linear programs.

\subsection{Image reconstruction with solution functions}
We consider an original image of $256\times 256$ and partition on its domain (coordinate-axes) using triangular partitions. In each partition, an affine function represents the RGB pixel values of the image at each corner of the domain. This function takes the input as x,y-position (for 2-dimensional images) and outputs RGB pixel values. We reconstruct each color channel separately. Now, we reconstruct this piecewise affine function as a LP solution function using CVX Matlab combination \cite{GrantBoyd_cvx, GrantBoyd_recent}. The reconstructed solution function corresponds to an LP with 2 variables and a total number of 130,050 constraints. Total computation time at $100\%$ complexity is 47 seconds.

Now, we investigate the approximation capability by removing a random subset of the inequalities in the LP; we then reconstruct the image as the solution function of the new LP (with reduced complexity). In Figure \ref{fig:image_decompose} and Table \ref{tab:number-constr}, we report complexity (number of constraints) as the percentage of the original constraints. Mean squared error (MSE) measures the difference between the original image and the reconstructed image.  We observe that MSE increases with decreasing complexity. However, it is interesting to see that the visual quality has only begun to decline beyond 90\% reduction.

\begin{table}[h]
\centering
\begin{tabular}{c|ccccc}
\textbf{}            & \multicolumn{5}{c}{\textbf{complexity}} \\ 
                     & Original        & 50\%         & 20\%&10\%&5\%         \\\hline
\textbf{\# of constraints} & 130050  & 65025  & 26010 &13000&6500\\
\textbf{MSE} & 0  & 0.0078  & 0.008 &0.026&0.061
\end{tabular}\caption{Complexity and MSEs of the reconstructed solution functions.}\label{tab:number-constr}
\end{table}

 \begin{figure*}[h!]
 \centering
 \begin{subfigure}[b]{0.19\linewidth}
 			\includegraphics[width=\linewidth]{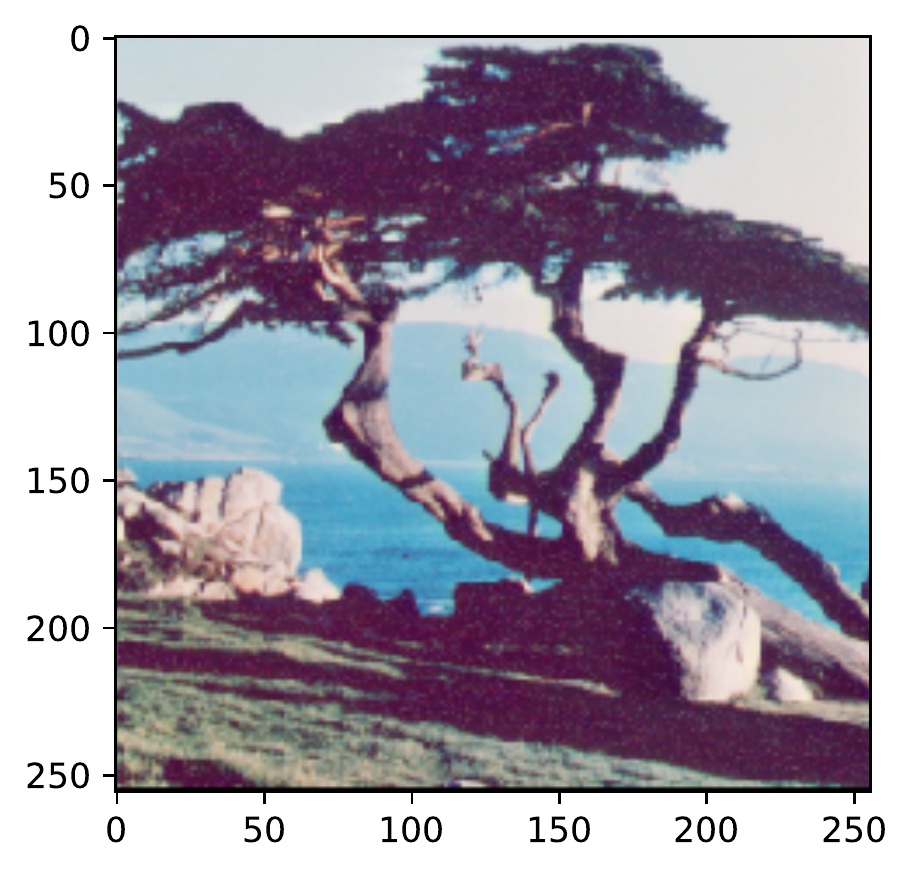}
 			\caption{Original image.}
 		\end{subfigure}
 		\begin{subfigure}[b]{0.19\linewidth}
 			 \includegraphics[width=\linewidth]{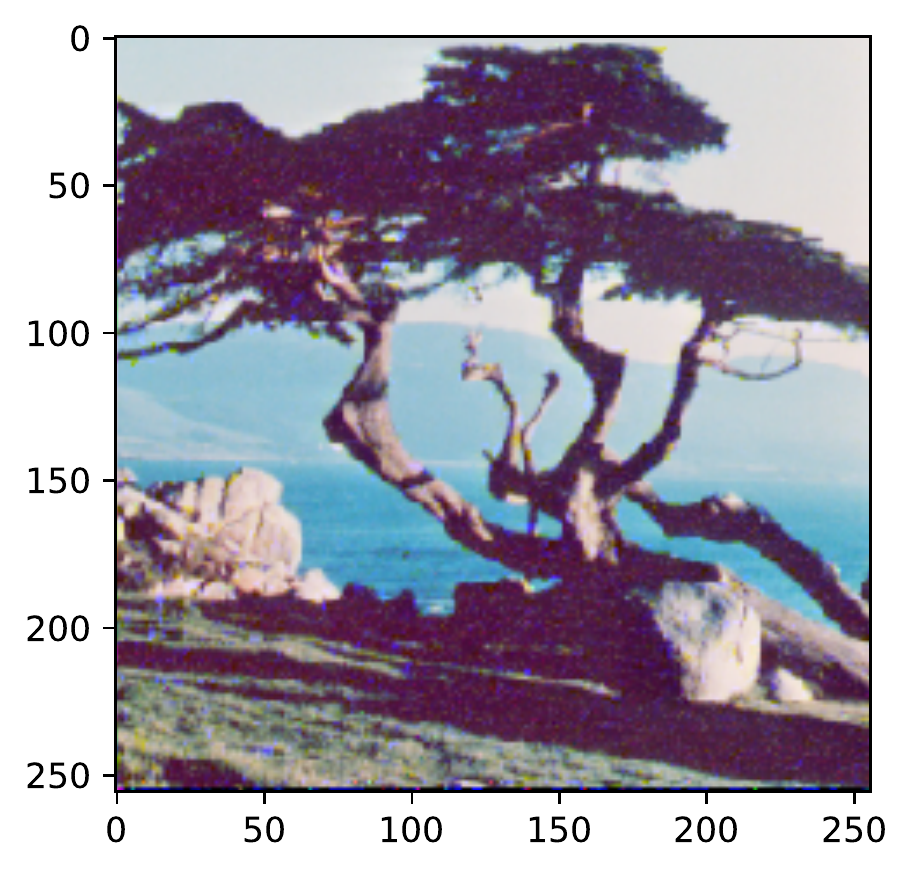}
 			\caption{$\text{complexity: }50\%$}
 		\end{subfigure}
 		\begin{subfigure}[b]{0.19\linewidth}
 			\includegraphics[width=\linewidth]{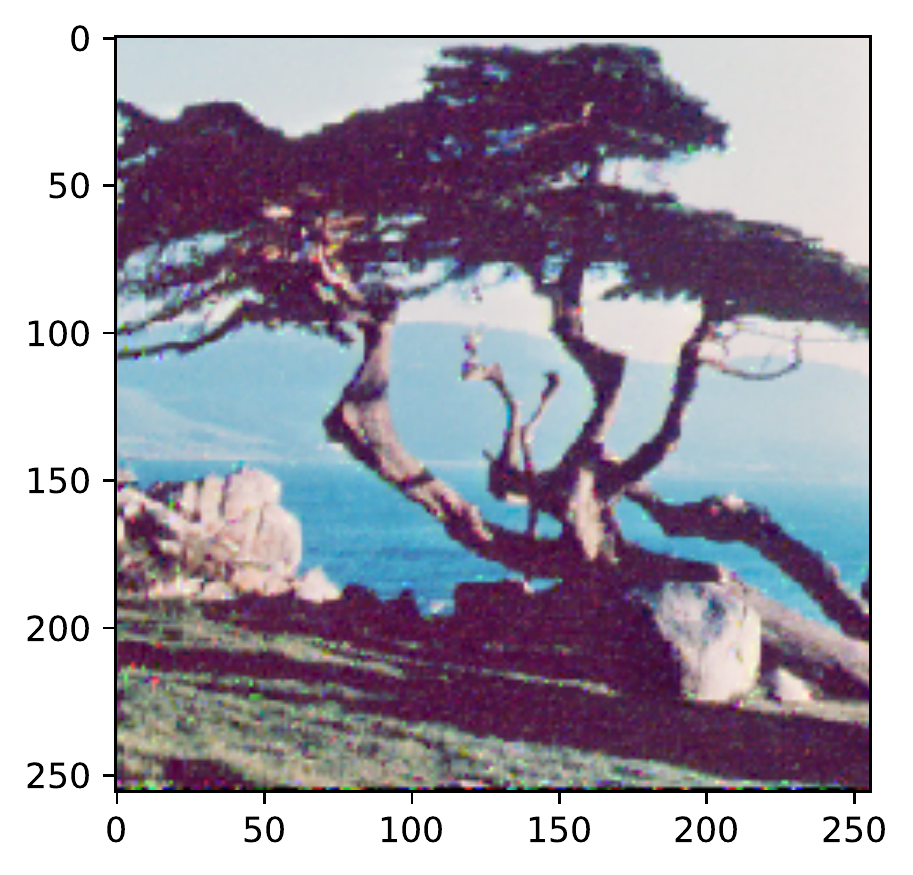}
 			\caption{$\text{complexity: }20\%$}
 		\end{subfigure}
 		\begin{subfigure}[b]{0.19\linewidth}
 			\includegraphics[width=\linewidth]{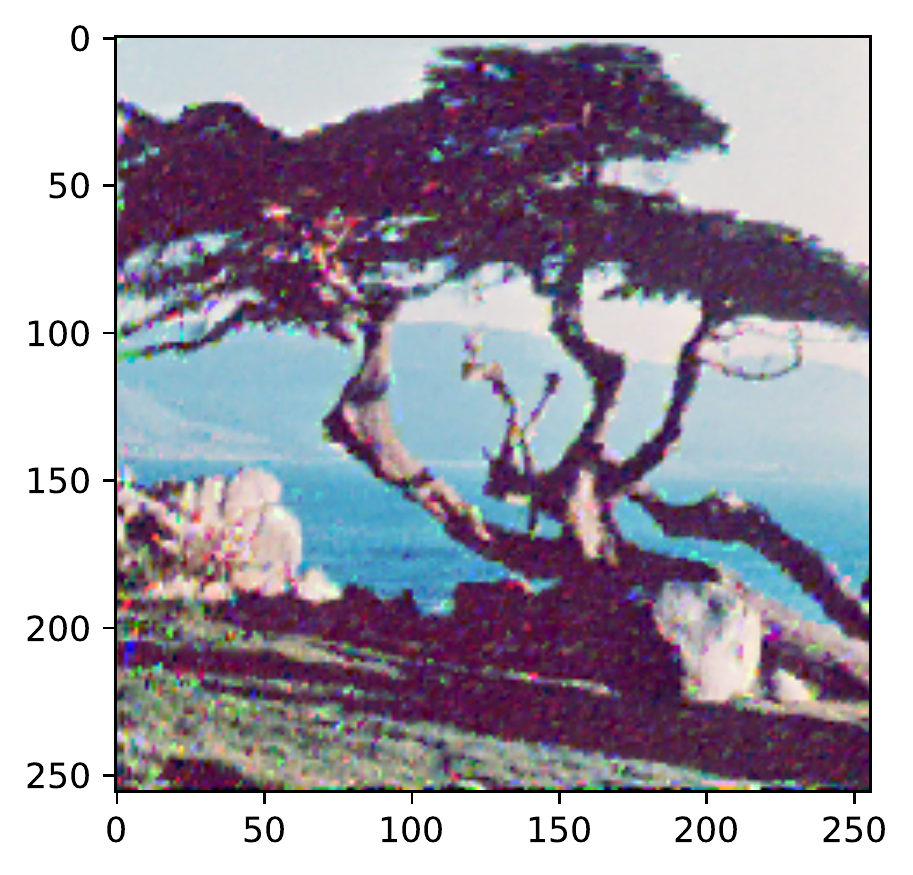}
 			\caption{$\text{complexity: }10\%$}
 		\end{subfigure}
 		\begin{subfigure}[b]{0.19\linewidth}
 			\includegraphics[width=\linewidth]{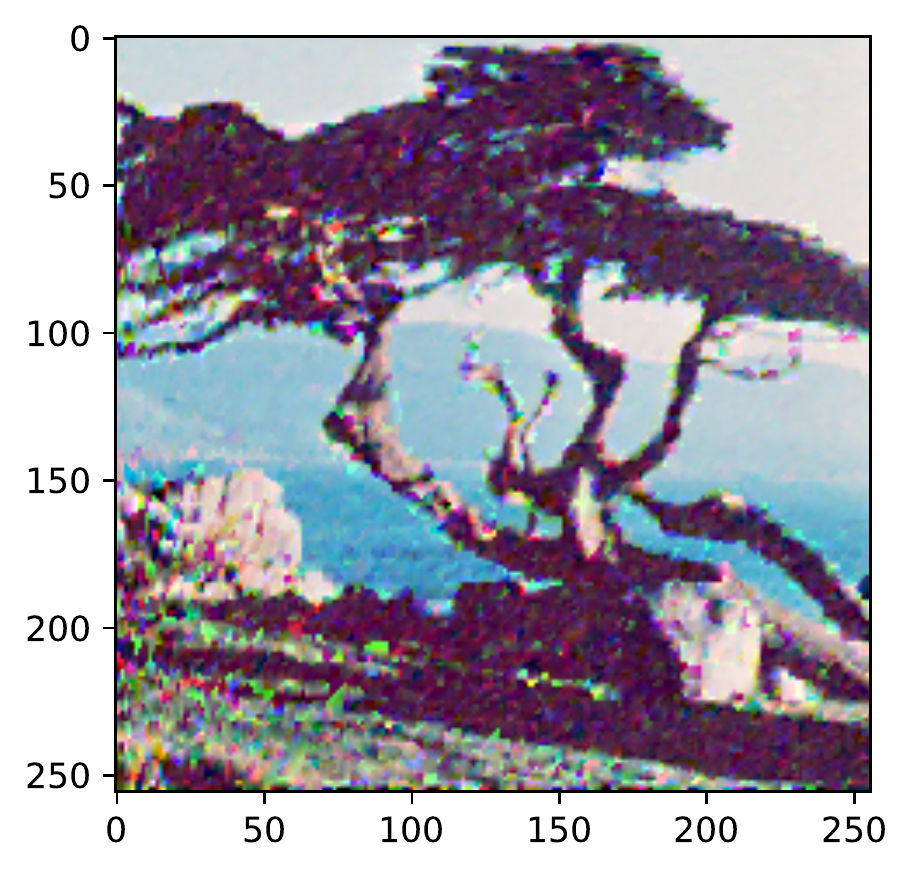}
 			\caption{$\text{complexity: }5\%$}
 		\end{subfigure} 
 		\caption{\footnotesize Image reconstruction with respect to changing complexity (number of constraints).}	
 		\label{fig:image_decompose}
 	\end{figure*}


\subsection{Reconstruction of SciPy test functions}
For the second experiment, we verify the approximation power of layered architecture on benchmark test functions \footnote{\url{http://infinity77.net/go_2021/scipy_test_functions.html#scipy-test-functions-index}}. For ease of presentation, in this experiment, we have selected three 2-dimensional functions, namely, Alpine, Parsopoulos, TridiagonalMatrix, and a four-dimensional function, i.e., Powell. For these functions, we also report the reconstruction error by increasing the number of partitions ($N$) as in Theorem \ref{thm:multilayer_n}. Note that these functions belong to $C^\infty$, which is a superset of $C^k$ for any $k$; we report the reconstruction error for different smoothing orders $k$. From Figures \ref{fig:Alpine}, \ref{fig:parsopoulos}, \ref{fig:TridiagonalMatrix}, and \ref{fig:Powell}, we can accurately reconstruct the test functions. The MSE plots show that the error decreases as the number of partitions ($N$) increases.

In Figure \ref{fig:Alpine}, we reconstruct an objective function of Alpine multimodal minimization problem, defined for $n_x$-dimension input $x\in\mathbb{R}^{n_x}$ as follows $$f(x) = \prod_{i=1}^{n_x}\sqrt{x_i} \sin{x_i} .$$ For ease of representation, we reconstruct this in 2-dimension $(n_x=2)$ for the inputs $x_i\in[0,10]$ for $i = 1,2$. Figure \ref{fig:Alpine} on the left shows the original and reconstructed functions, which look almost identical; on the right is the decreasing MSE with increasing partitions $N$. We have reconstructed this for smoothness orders $k \in \{3, \dots, 10\}$. The higher the order, the fewer partitions are needed to achieve the same accuracy. Note that the reconstruction of the first plot in Figure \ref{fig:Alpine} is for maximum smoothness $k = 6$.  
\begin{figure}[H]
	\centering
	\includegraphics[width=\linewidth]{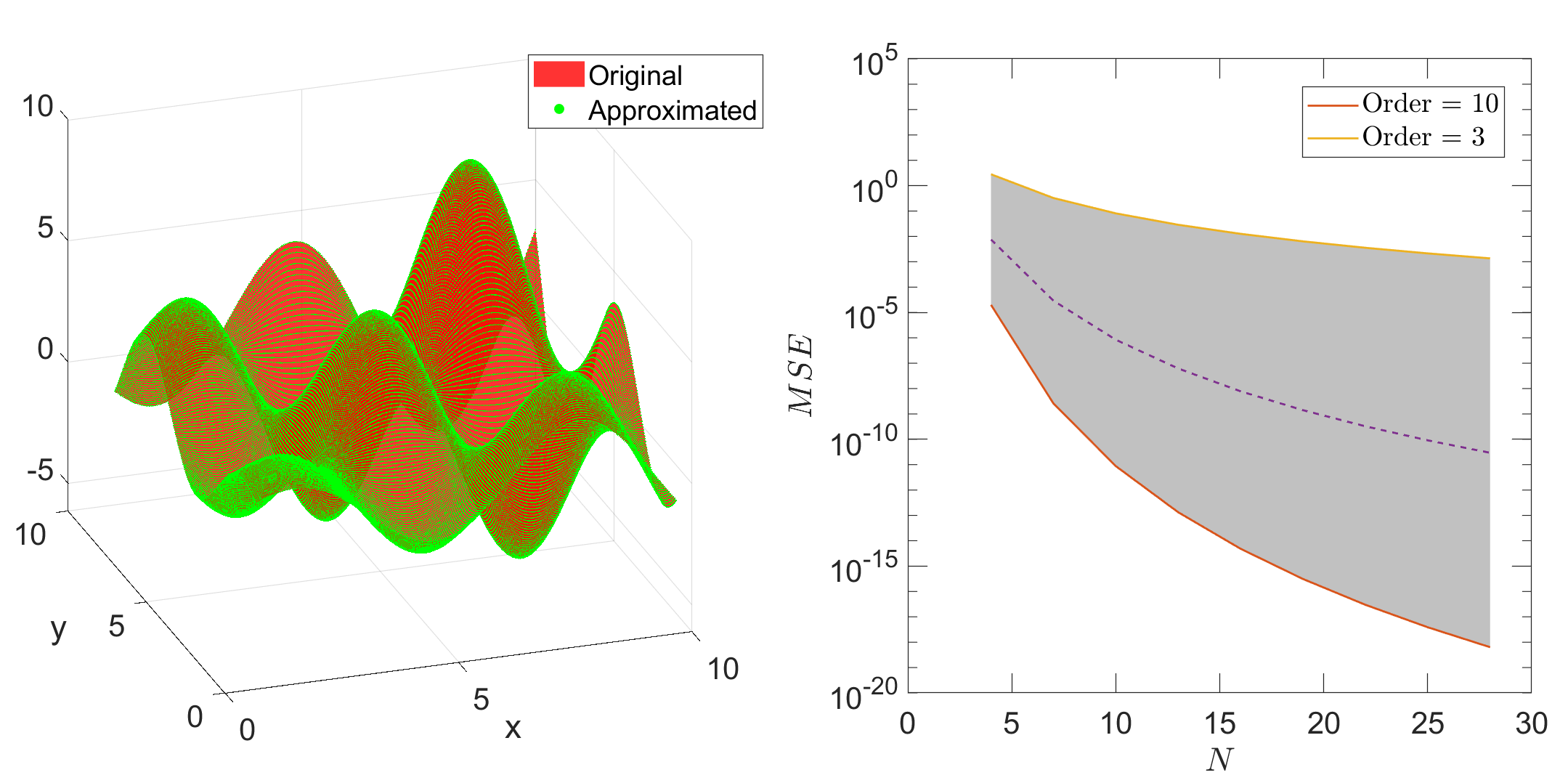}
	\caption{\footnotesize Reconstruction of Alpine. \label{fig:Alpine}}
\end{figure}

In Figure \ref{fig:parsopoulos}, we reconstruct  an objective function of a Parsopoulos multimodal minimization problem,
$$f(x) = \cos(x_1)^2 + \sin(x_2)^2.$$
We solve this problem for $x_i\in[-5, 5]$ for $i = 1, 2$. In Figure \ref{fig:parsopoulos}, the first plot is the reconstruction for smoothness $k = 6$ and the second is the decreasing MSE error with increasing partitions $N$ for maximum smoothness $k \in \{4,\dots, 7\}$. 
 	\begin{figure}[H]
 		\centering
 		\includegraphics[width=\linewidth]{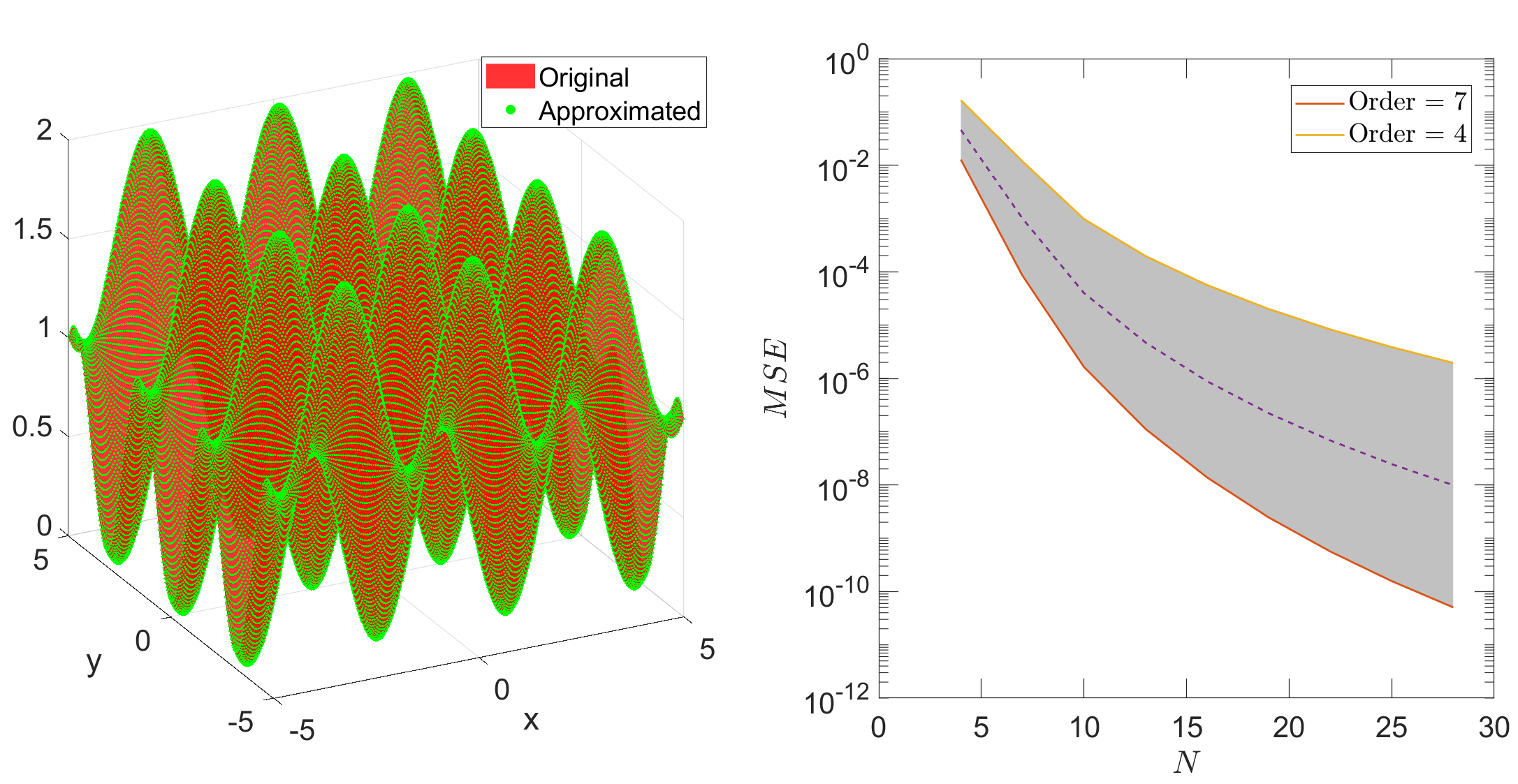}
 		\caption{\footnotesize Reconstruction of Parsopoulos. \label{fig:parsopoulos}}
 	\end{figure}
 	
 	In Figure \ref{fig:TridiagonalMatrix}, we  reconstruct an objective function of Trid multimodal minimization problem,
 	$$f(x) = \sum_{i=1}^k (x_i - 1)^2 - \sum_{i=1}^k (x_i - 1)^2. $$ The function is reconstructed for $x_i\in [-20, 20]$ for $i \in \{1, 2\}$. 
 	\begin{figure}[H]
 		\includegraphics[width=\linewidth]{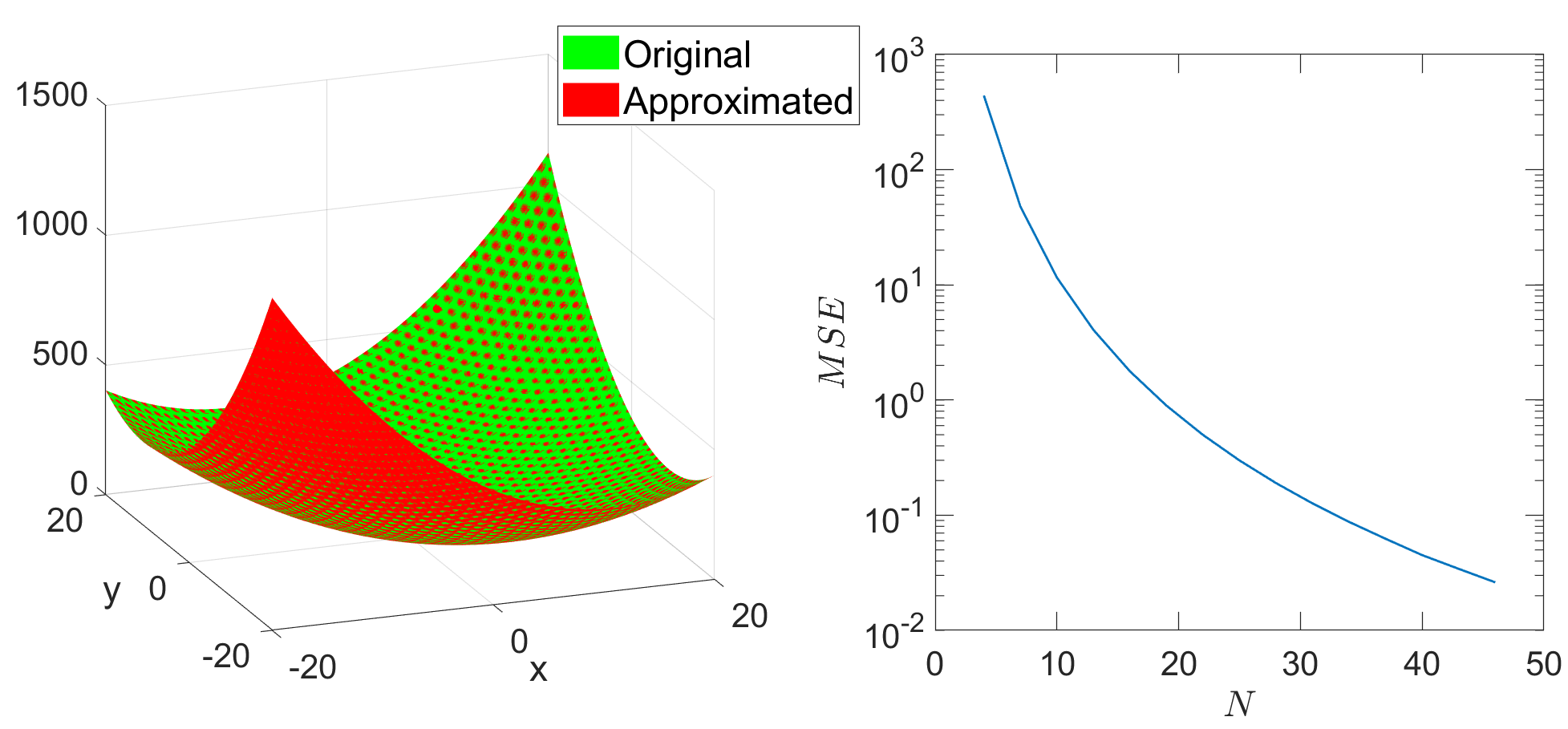}
 		\caption{\footnotesize Reconstruction of TridiagonalMatrix. \label{fig:TridiagonalMatrix}}
 	\end{figure}
 	  For Figure \ref{fig:Powell}, we consider an objective function of multimodel optimization problem, Powell in 4-D, as follows	 	
 	  $$f(x) = (x_3-10x_1)^2+ 5(x_2-x_4)^2 + (x_1-2x_2)^4 + 10 (x_3-x_4)^4.$$
 	   	 We reconstruct the above for $x_i\in [-4, 5],$ for $i \in \{1, \dots, 4\}.$ Note that, this is a 4-D function and, for representation, we only provide MSE variation  with changing the number of partitions $N$. 
 	\begin{figure}[H]
 		\centering
 		\includegraphics[width=0.5\linewidth]{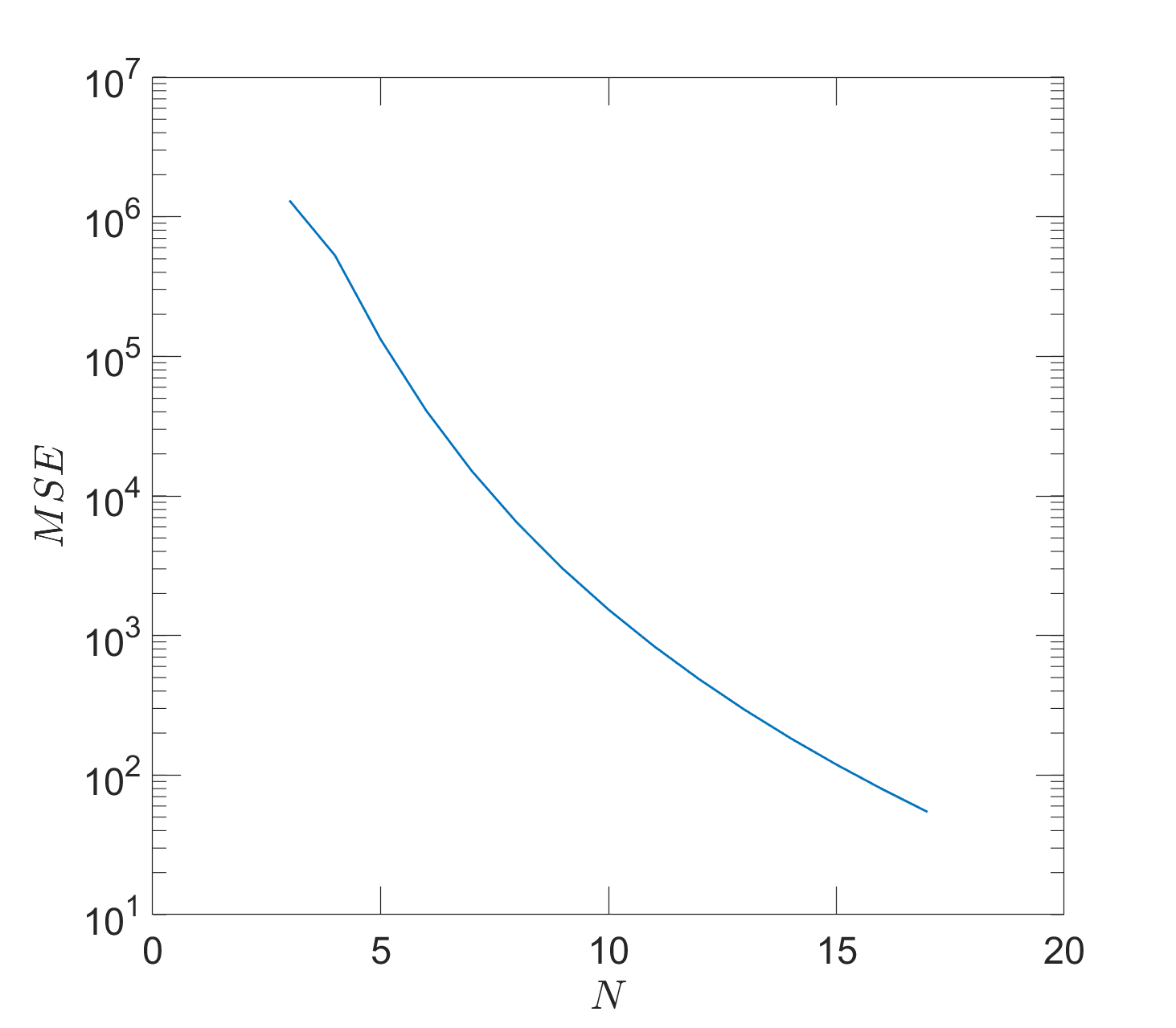}
 		\caption{\footnotesize MSE of the reconstruction for Powell (4-D).\label{fig:Powell}}
 	\end{figure}

\end{document}